\documentclass{article}
\PassOptionsToPackage{numbers, compress}{natbib}
\usepackage{natbib}
\usepackage{graphicx} %
\usepackage{fancyhdr}

\usepackage{fullpage}%
\usepackage{subcaption}

\usepackage[utf8]{inputenc} %
\usepackage[T1]{fontenc}    %
\usepackage[hidelinks]{hyperref}       %
\usepackage{url}            %
\usepackage{booktabs}       %
\usepackage{amsfonts}       %
\usepackage{nicefrac}       %
\usepackage{microtype}      %
\usepackage{xcolor}         %
\usepackage{tikz}
\usepackage{amssymb}
\usepackage{amsthm}
\usepackage{amsmath}
\usepackage{cleveref}

\usepackage{mathtools}
\usepackage{enumitem}
\usepackage{xspace}
\usepackage{wrapfig}
\usepackage{algorithm}
\usepackage{algorithmic}
\usepackage{authblk}

\usepackage{siunitx}
\usepackage{amartya_ltx}
\usepackage{multirow}
\usepackage{thm-restate}

\usepackage{titlesec}
\usepackage[acronym]{glossaries}  %

\newacronym{ml}{ML}{Machine Learning}
\newacronym{id}{ID}{in-distribution}
\newacronym{ood}{OOD}{out-of-distribution}
\newacronym{fmow}{fMoW}{Functional Map of the World}
\newcommand{\aline}{\emph{Accuracy-on-the-line}\xspace}
\newcommand{\awline}{\emph{Accuracy-on-the-wrong-line}\xspace}

\title{Accuracy on the wrong line: On the pitfalls of noisy data for out-of-distribution generalisation}

\author[1]{Amartya Sanyal}
\author[1]{Yaxi Hu}
\author[2]{Yaodong Yu}
\author[3]{Yian Ma}
\author[4]{Yixin Wang}
\author[1]{Bernhard Sch\"{o}lkopf}

\affil[1]{Max Planck Institute for Intelligent Systems, T\"ubingen, Germany}
\affil[2]{University of California, Berkeley, U.S.A.}
\affil[3]{Hal\i c\i o\u{g}lu Data Science Institute, University of California San Diego, San Diego, U.S.A.}
\affil[4]{University of Michigan, Ann Arbor, U.S.A.}

\date{}

\begin{document}
\maketitle

\begin{abstract}
``\aline'' is a widely
observed phenomenon in machine learning, where a model's accuracy on
in-distribution (ID) and out-of-distribution (OOD) data is positively
correlated across different hyperparameters and data configurations.
But when does this useful relationship break down? In this work, we
explore its robustness. The key observation is that noisy data and the
presence of nuisance features can be sufficient to shatter the~\aline phenomenon. In these cases, ID and OOD accuracy can become
negatively correlated, leading to ``\awline.'' This
phenomenon can also occur in the presence of spurious (shortcut)
features, which tend to overshadow the more complex signal (core,
non-spurious) features, resulting in a large nuisance feature space.
Moreover, scaling to larger datasets does not mitigate this
undesirable behavior and may even exacerbate it. We formally prove a lower bound on~\Gls{ood} error in a linear classification
model, characterizing the conditions on the noise and nuisance features for a large~\Gls{ood} error. We finally demonstrate this phenomenon across both
synthetic and real datasets with noisy data and nuisance features.

\end{abstract}

\section{Introduction}
\label{sec:intro}

\begin{wrapfigure}{r}{0.6\linewidth}
    \begin{subfigure}[t]{0.49\linewidth}
    \includegraphics[width=1.0\linewidth]{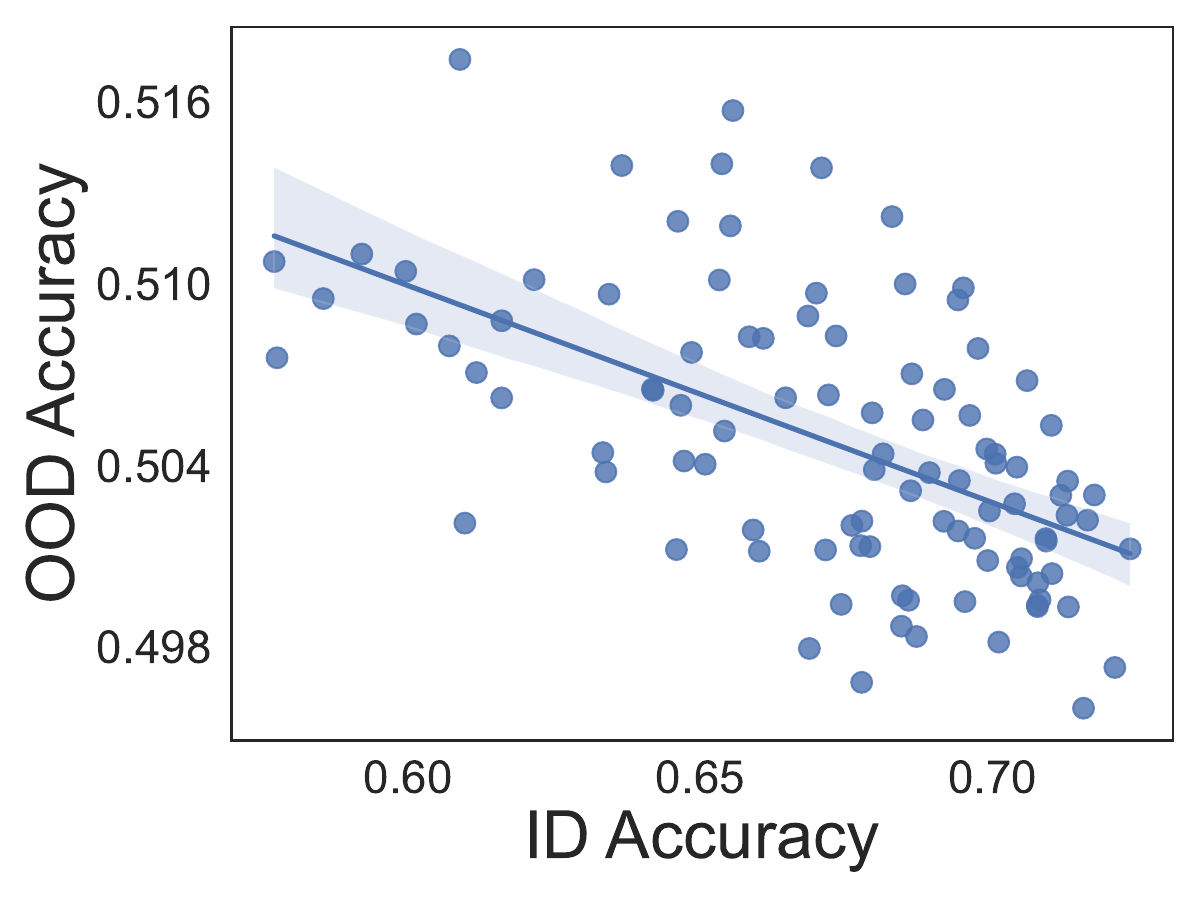}
    \caption{Noisy dataset}
    \end{subfigure}
    \begin{subfigure}[t]{0.49\linewidth}
    \includegraphics[width=1.0\linewidth]{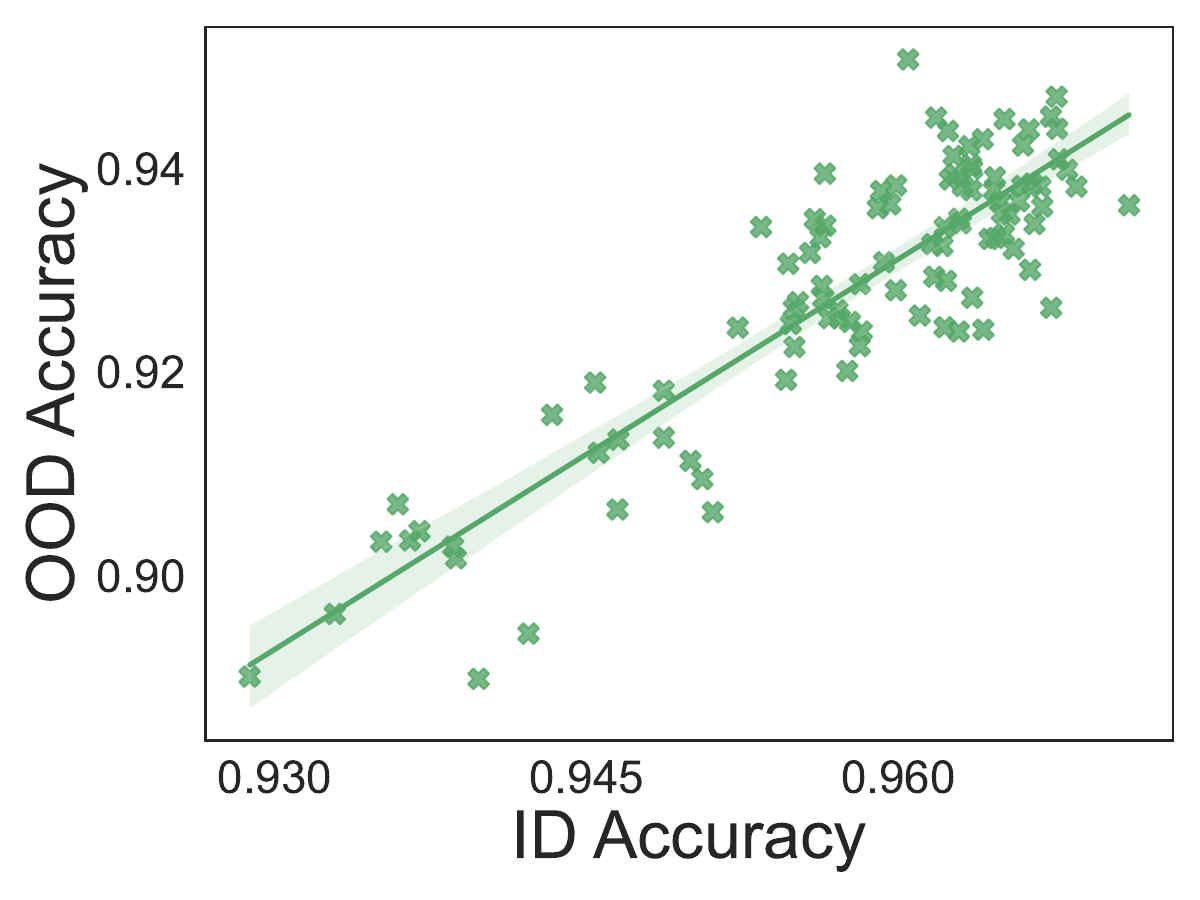}
    \caption{Noiseless dataset}
    \end{subfigure}
    \caption{\small~\awline behaviour in Noisy dataset vs.~\aline behaviour in Noiseless dataset in linear setting. See~\Cref{sec:linear} for a description of the setting.}
    \label{fig:lin-id-vs-ood}
\end{wrapfigure}

\Gls{ml} models often exhibit a consistent behavior known as
``\aline''~\citep{miller2021accuracy}. This phenomenon
refers to the positive correlation between a model's accuracy
on \Gls{id} and \Gls{ood} data. The positive correlation is widely observed across various models, datasets,
hyper-parameters, and configurations.  It suggests that  improving a model's ID performance can enhance its
generalization to OOD data. This addresses a fundamental challenge in
machine learning: extrapolating knowledge to new, unseen scenarios by allowing OOD performance assessment across models
without retraining. It also suggests that modern machine learning does
not need to trade off~\Gls{id} and \Gls{ood} accuracy.

However, does~\aline always hold? In this work,
we explore the robustness of this phenomenon.  Our key observation,
supported by theoretical insights, is that~\emph{noisy data and the presence of nuisance features can shatter the phenomenon, leading to an
\emph{\awline} phenomenon (\Cref{fig:lin-id-vs-ood})
with a negative correlation between ID and OOD accuracy}. Noisy
data is common in machine learning, as datasets expand and are sourced
automatically from the web, introducing label noise through human
annotation \citep{frenay2014class,northcutt2021confident}. It is also common for modern~\Gls{ml} models to obtain zero training error on noisy training data~\citep{zhang2017understanding}, a phenomenon called noisy interpolation. In this work, we show that when algorithms memorise noisy data, the correlation between ID and OOD
performance can become nearly inverse.

Beyond noisy data, another crucial condition for
\awline is the presence of multiple ``nuisance features'', namely features irrelevant to the classification task.
Nuisance features are common in machine learning, as task-relevant
features in high-dimensional data frequently
lies on a low-dimensional manifold; see e.g.
\citet{brown2022verifying} and \citet{pope2021intrinsic}, This lower
intrinsic dimensionality implies that classification-relevant
information is concentrated in a smaller, more manageable subset of
the feature space, rendering the remaining features as ``nuisance''.%

Even in the direct absence of these nuisance features, this phenomenon can occur due to so-called spurious features. These are features that are not genuinely relevant
to the target task but appear predictive because of coincidental
correlations or dataset biases.  Spurious features are often simpler
and more easily learned than non-spurious ones, creating the illusion that data lies on a low-dimensional manifold defined by these spurious features. This makes other non-spurious features effectively ``nuisance''. This illusion often becomes reality in training and has been observed in a series of works~\citep{shah2020pitfalls,arjovsky2019invariant,parascandolo2021learning,singla2021salient}\footnote{This observation dates back to the `urban tank legend' e.g. see section 8 in~\citet{scholkopf2022causality}}:
Models tend to exploit the easiest spurious features during training, overshadowing the true, more complex non-spurious ones. This leads to a large nuisance space that exceeds the true number of nuisance features, as observed in~\citet{qiu2023complexity}.

\begin{wrapfigure}{r}{0.6\linewidth}
    \begin{subfigure}[t]{0.48\linewidth}
    \includegraphics[width=1.0\linewidth]{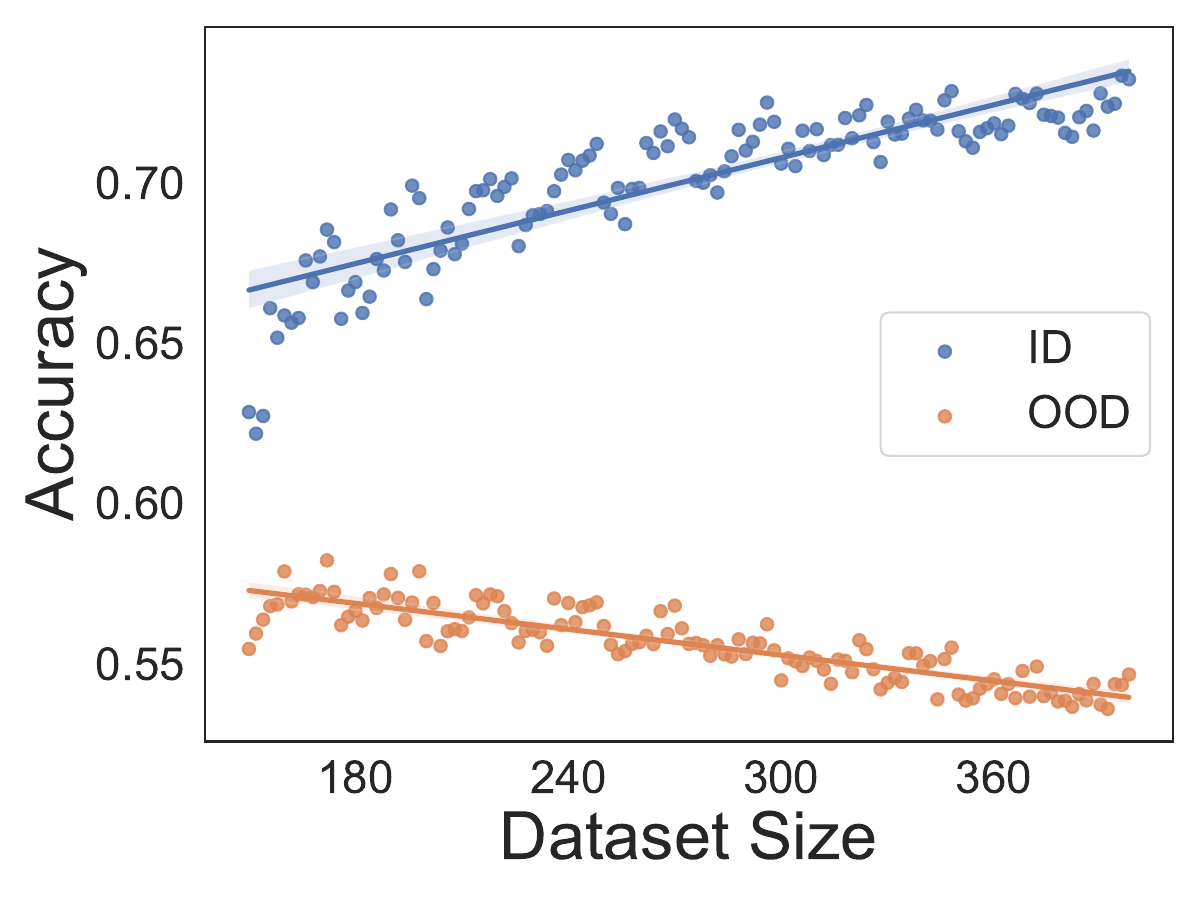}
    \caption{Noisy dataset}
    \end{subfigure}
    \begin{subfigure}[t]{0.48\linewidth}
    \includegraphics[width=1.0\linewidth]{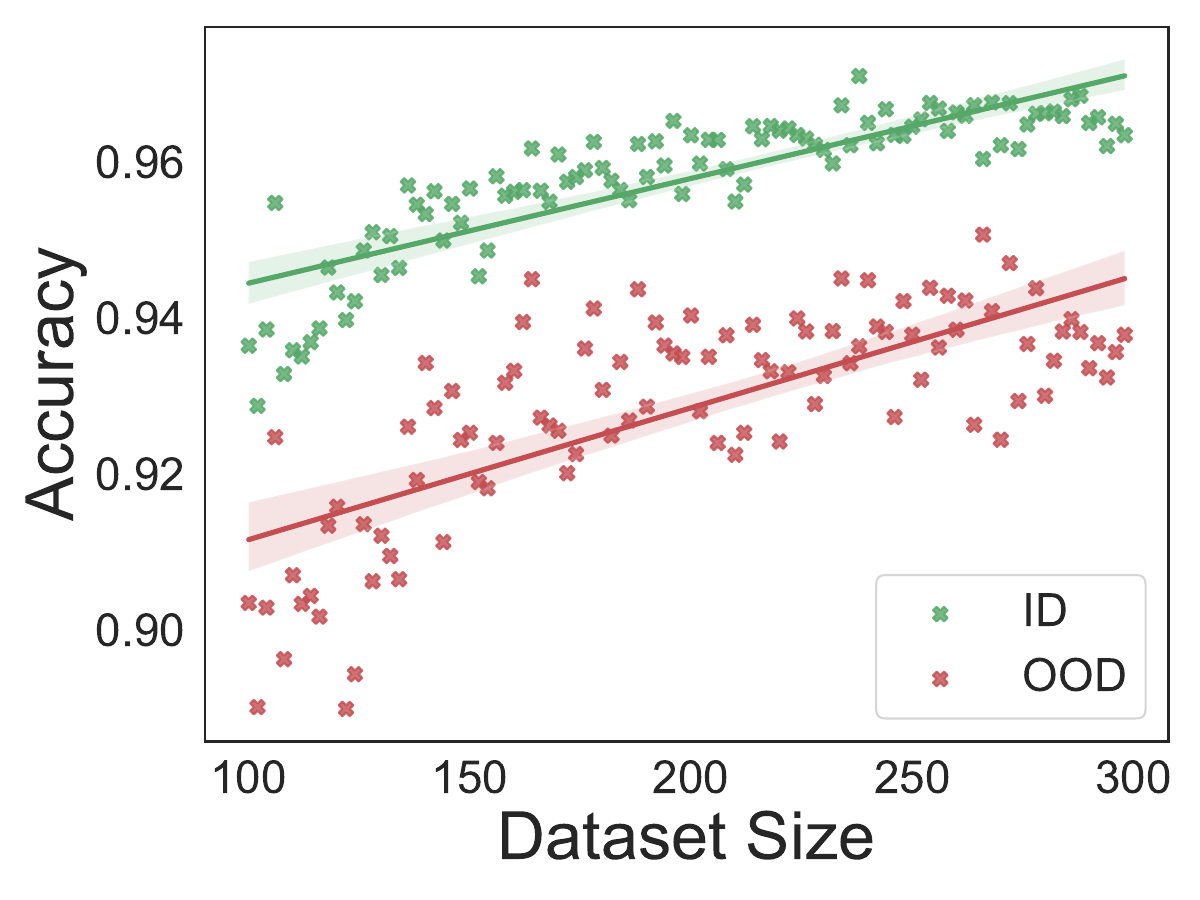}
    \caption{Noiseless dataset}
    \end{subfigure}
    \caption{\small Same setting as~\Cref{fig:lin-id-vs-ood}, increasing dataset size always increases~\Gls{id} accuracy irrespective of label noise, but decreases \Gls{ood} accuracy in the presence of label noise.}
    \label{fig:diff-lin-id-vs-ood}
\end{wrapfigure}
A reader might ask: Can scaling up dataset size resolve the
undesirable \awline phenomenon? Recent literature
on scaling laws \citep{kaplan2020scaling,hestness2017deep} as well as uniform-convergence-based results in classical learning theory suggest that larger datasets are needed to fully benefit from
larger models. Even with noise interpolation, larger datasets usually improve generalisation error~\citep{bartlett2020benign, belkin2019reconciling,nakkiran2021deep}.  However, as~\Cref{fig:diff-lin-id-vs-ood} suggests, we show that the answer is no. Scaling
can adversely impact \gls{ood} error, exacerbating the negative
correlation between ID and OOD performance. As dataset sizes grow,
even a small label noise rate increases the absolute number of noisy
points, significantly impacting \Gls{ood} error, even if it may not
always affect \Gls{id} error.

\textbf{Contributions.} To summarize, the contributions of this work are as follows: (1) We show that~\awline can occur in practice and provide experimental results on two image datasets. (2) In a linear setting, we formally prove a lower bound on the difference between per-instance \Gls{ood} and \Gls{id} error and characterise sufficient conditions for it to increase. (3) We argue that these conditions are natural properties of learned models, especially for noisy interpolation and in the presence of nuisance features.

\textbf{Related work.} Our work builds on the body of work
on the~\aline phenomenon.~\citet{miller2021accuracy}
first demonstrated empirically that~\Gls{id} and~\Gls{ood} performances exhibit a linear correlation, across many datasets and associated distribution shifts, while~\citet{abe2022} showed a similar positive correlation for ensemble models.~\citet{pmlr-v202-liang23d} zoomed in on subpopulation shift--a particular type of distribution shift---and observed a nonlinear "accuracy-on-the-curve" phenomenon. Theoretical analyses by \citet{tripuraneni2021covariate,tripuraneni2021} supported
\aline in overparameterised linear models.
\citet{baek2022} extended the concept to "agreement-on-the-line,"
showing a correlation between OOD and ID agreement for neural
network classifiers.~\citet{pmlr-v202-lee23o} provided theoretical
analysis for this phenomenon in high-dimensional random features
regression.~\citet{kim2023reliable} used these insights to develop a
test-time adaptation algorithm for robustness against distribution
shift while~\citet{eastwood2024spuriosity} used similar insights to develop a pseudo-labelling strategy to develop robustness across domain shifts. In contrast to these works, our work focuses on the robustness
of~\aline, investigating when it breaks and under what
conditions.

Other works have also studied the robustness of~\aline.~\citet{wenzel2022} and~\citet{teney2023} empirically suggested that it holds in most but not all datasets and configurations.In
particular,~\citet{teney2023} also provided a simple linear example
where adding spurious features can differently impact ID and OOD
risks. \citet{kumar2022finetuning} theoretically established an ID-OOD
accuracy tradeoff when fine-tuning overparameterised two-layer linear
networks. In contrast to all of these works, our work provides both a
theoretical analyses and empirical investigations demonstrating the
essential impact of \emph{label noise} and \emph{nuisance features} in
breaking \aline. Moreover, we state that it leads to a
negative correlation between ID and OOD accuracies, dubbed
"\awline." %
Other lines of work show how label noise affects adversarial
robustness~\citep{sanyal2021how,paleka2023a} and
fairness~\citep{wu2022fair,wang2021fair}. However, they are not directly related to this work.

\section{``Accuracy-on-the-wrong-line'' in practice}
\label{sec:real-world}

We first show how a combination of practical limitations in modern machine learning can result in~\awline in two real-world computer vision datasets: MNIST~\citep{mnist} and \Gls{fmow}~\citep{christie2018functional}. Then in~\Cref{sec:theory}, we formalise the necessary conditions and provide a theoretical proof showing their sufficiency. In~\Cref{sec:linear}, we conduct synthetic interventional experiments in a simple linear setting to demonstrate these conditions are indeed sufficient and behave in line with our theoretical results.

\subsection{Colored MNIST dataset}
\label{sec:cmnist}
\begin{figure*}[t]
    \centering
\begin{subfigure}[t]{0.24\linewidth}
    \includegraphics[width=0.99\linewidth]{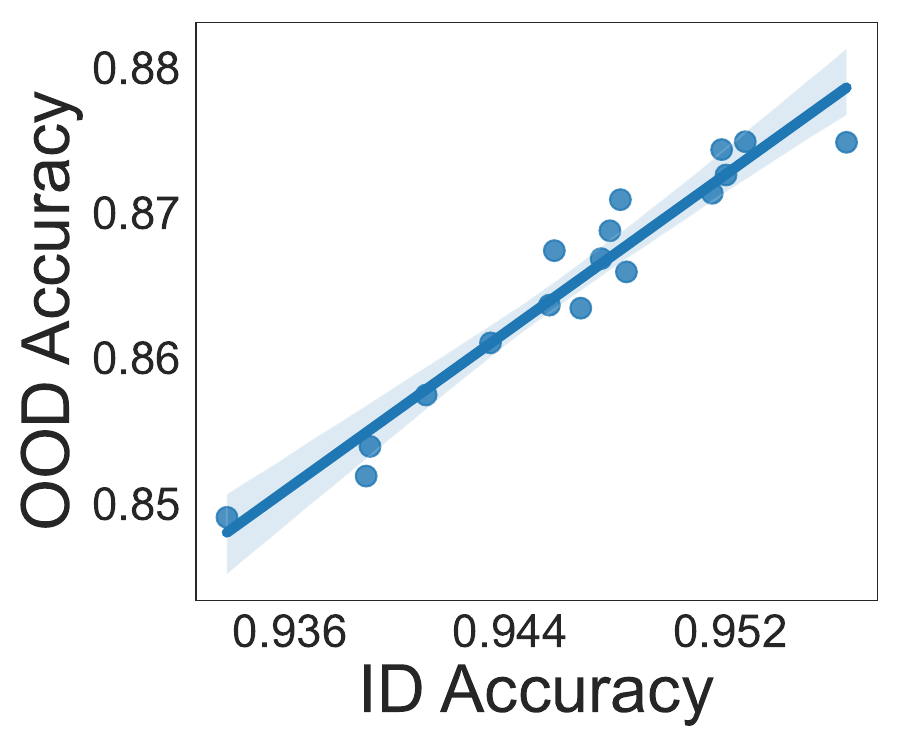}
    \caption{\(\eta=0.\)}
  \label{fig:mnist0}
\end{subfigure}
\begin{subfigure}[t]{0.24\linewidth}
    \includegraphics[width=0.99\linewidth]{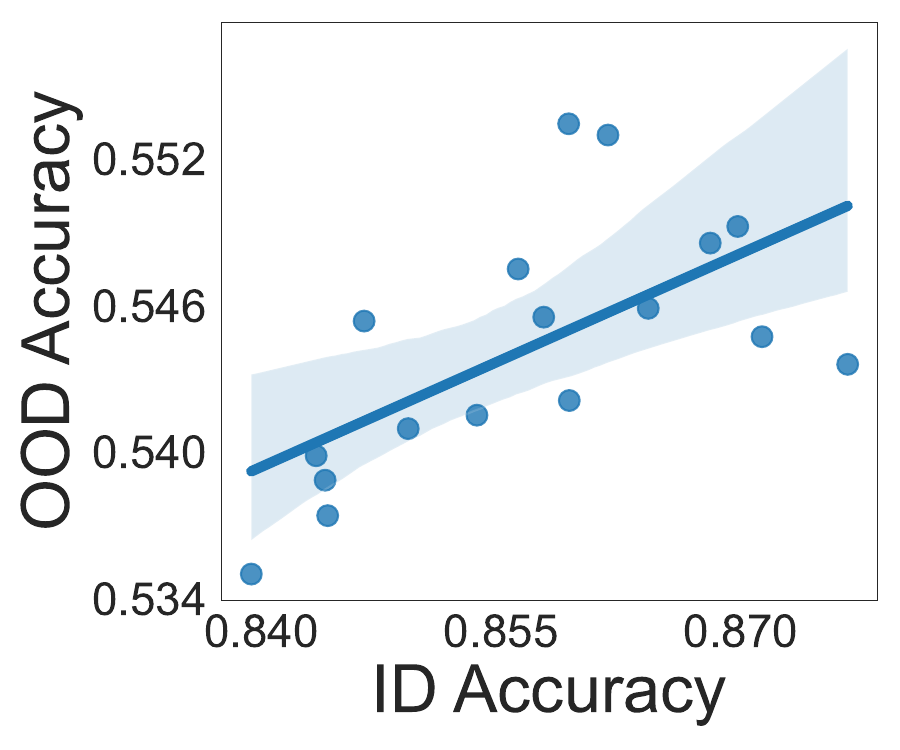}
    \caption{\(\eta=0.15\)}
    \label{fig:mnist15}
\end{subfigure}
\begin{subfigure}[t]{0.24\linewidth}
    \includegraphics[width=0.99\linewidth]{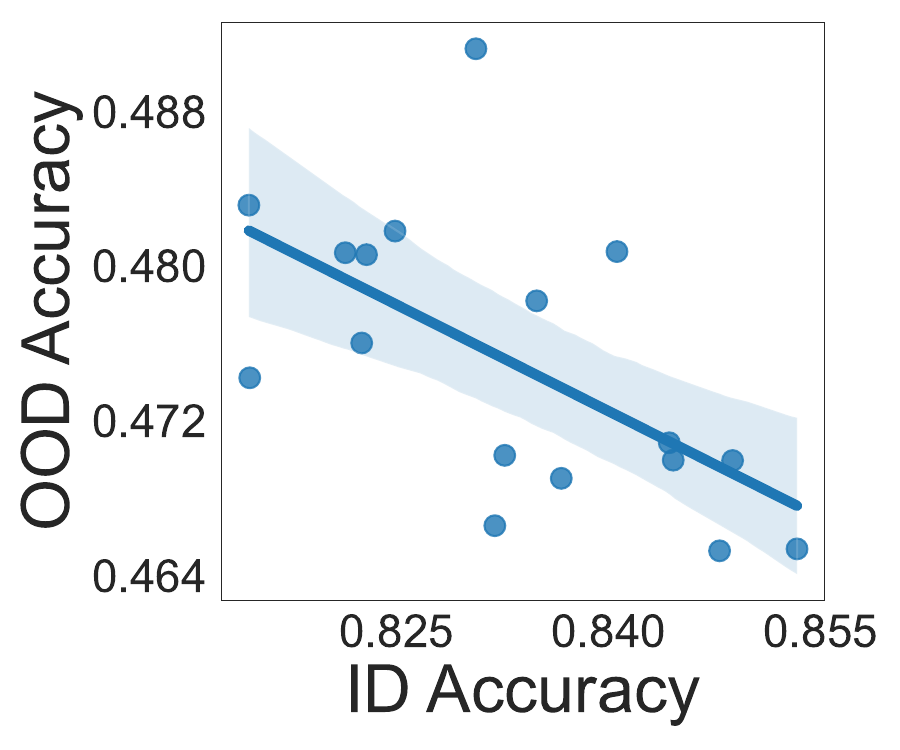}
    \caption{\(\eta=0.2\)}
    \label{fig:mnist2}
\end{subfigure}
\begin{subfigure}[t]{0.24\linewidth}
    \includegraphics[width=0.99\linewidth]{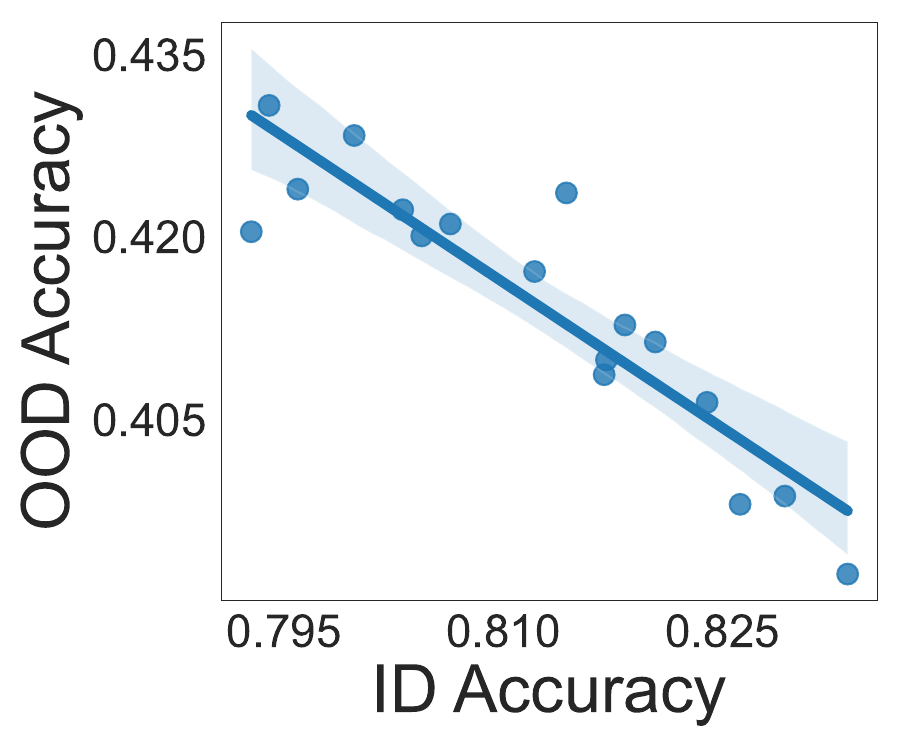}
    \caption{\(\eta=0.25\)}
    \label{fig:mnist25}
\end{subfigure}
\caption{\small Each plot shows~\Gls{ood} vs~\Gls{id} accuracy for varying label noise rates \(\eta\) on the colored MNIST dataset. Similar to~\Cref{fig:lin-id-vs-ood}, the~\aline phenomenon degrades with increasing amount of label noise.}
\label{fig:mnist}
\end{figure*}
We first examine the Colored MNIST dataset, a variant of MNIST, derived from MNIST by introducing a color-based spurious correlation, where the color of each digit is determined by its label with a certain probability. Specifically, digits are assigned a binary label based on their numeric value (less than 5 or not), which is then corrupted with label noise probability \(\eta\). The color assigned to each digit is thus correlated with the label with a small probability. A three-layer MLP is then trained on this dataset to achieve zero training error. The model is subsequently tested on a freshly sampled test set from the same distribution but without label noise and with a smaller spurious correlation; see~\Cref{app:cmnist} for more details. The accuracy on the training distribution is referred to as the~\Gls{id} accuracy, and the accuracy on the test distribution is referred to as the~\Gls{ood} accuracy.

The results of this experiment are presented in~\Cref{fig:mnist}. When the amount of label noise is low~(\Cref{fig:mnist0,fig:mnist15}), the~\Gls{id} and~\Gls{ood} accuracy are positively correlated, whereas they become negatively correlated at higher levels of label noise~(\Cref{fig:mnist2,fig:mnist25}).

\subsection{Functional Map of the World (fMoW) dataset}
\label{sec:fmow}

For the next set of experiments, we use the~\Gls{fmow}-CS dataset designed by~\citet{shi2023how} based on the original~\Gls{fmow} dataset~\citep{christie2018functional} in WILDS~\citep{wilds2021,sagawa2022extending}. The dataset contains satellite images from various parts of the world and are labeled according to one of 30 objects in the image. Similar to Colored MNIST,~\Gls{fmow}-CS dataset is constructed by introducing a spurious correlation between the geographic region and the label.  Similar to our previous experiments, we also introduce label noise with a probability of 0.5. For the~\Gls{ood} test data, we use the original WILDS~\citep{wilds2021,sagawa2022extending} test set for~\Gls{fmow}. %
Further details regarding the dataset are available in~\Cref{app:fmow}. To obtain various training runs, we fine-tuned ImageNet pre-trained models, including ResNet-18, ResNet-34, ResNet-50, ResNet-101, and DenseNet121, with various learning rates and weight decays on the~\Gls{fmow}-CS dataset. We also varied the width of the convolution layers to increase or decrease the width of each network. In total, we trained more than 400 models using various configurations and report the results in~\Cref{fig:fmow-exp-plot-main}. 

Consistent with previous experiments on Colored MNIST,~\Cref{fig:fmow-exp-plot-noise} shows that when the data comprises label noise and training accuracy is \(100\%\),~\Gls{id} and~\Gls{ood} accuracy are inversely correlated. In the absence of label noise,~\Cref{fig:fmow-exp-plot-noiseless} shows that the two are positively correlated. These two plots only consider models that fully interpolate the dataset: noisy and noiseless, respectively. To highlight that noisy interpolation is indeed necessary to break the~\aline phenomenon, we also plot the experiments for those training runs where the data is not fully interpolated in~\Cref{fig:fmow-exp-plot-noisy-nointerp}. This corresponds to early stopping, stronger regularizations, as well as smaller widths. Our results show that in this case, the~\Gls{id} and~\Gls{ood} accuracy are still positively correlated but to a lesser degree than the noiseless setting. We conjecture that this is because even minimizing the cross-entropy loss on noisy labels contributes to this behavior and is strongest when the minimisation leads to interpolation.

\begin{figure}[t]
\begin{subfigure}[t]{0.32\linewidth}
    \centering
    \includegraphics[width=0.99\linewidth]{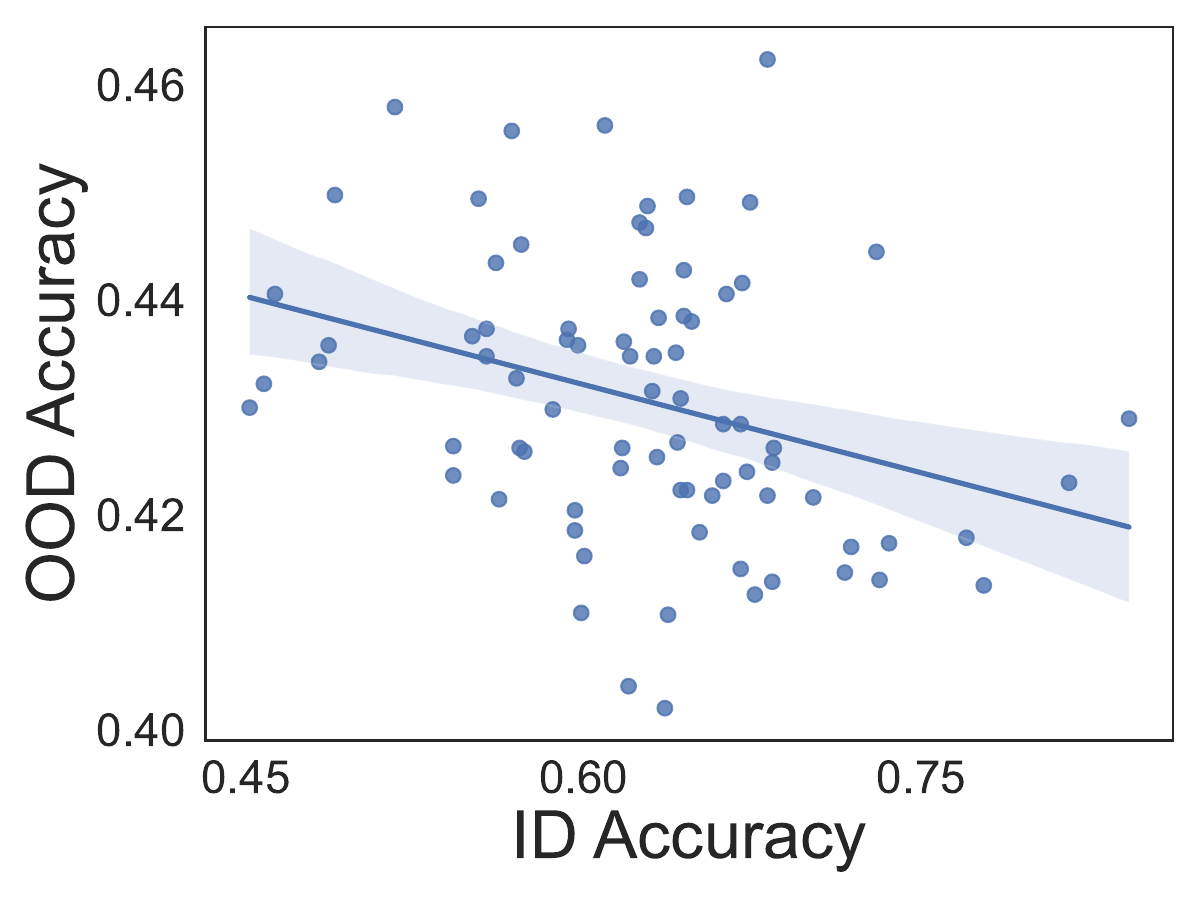}
    \caption{\footnotesize Noisy Interpolation}
    \label{fig:fmow-exp-plot-noise}
\end{subfigure}
\begin{subfigure}[t]{0.32\linewidth}
    \centering
    \includegraphics[width=0.99\linewidth]{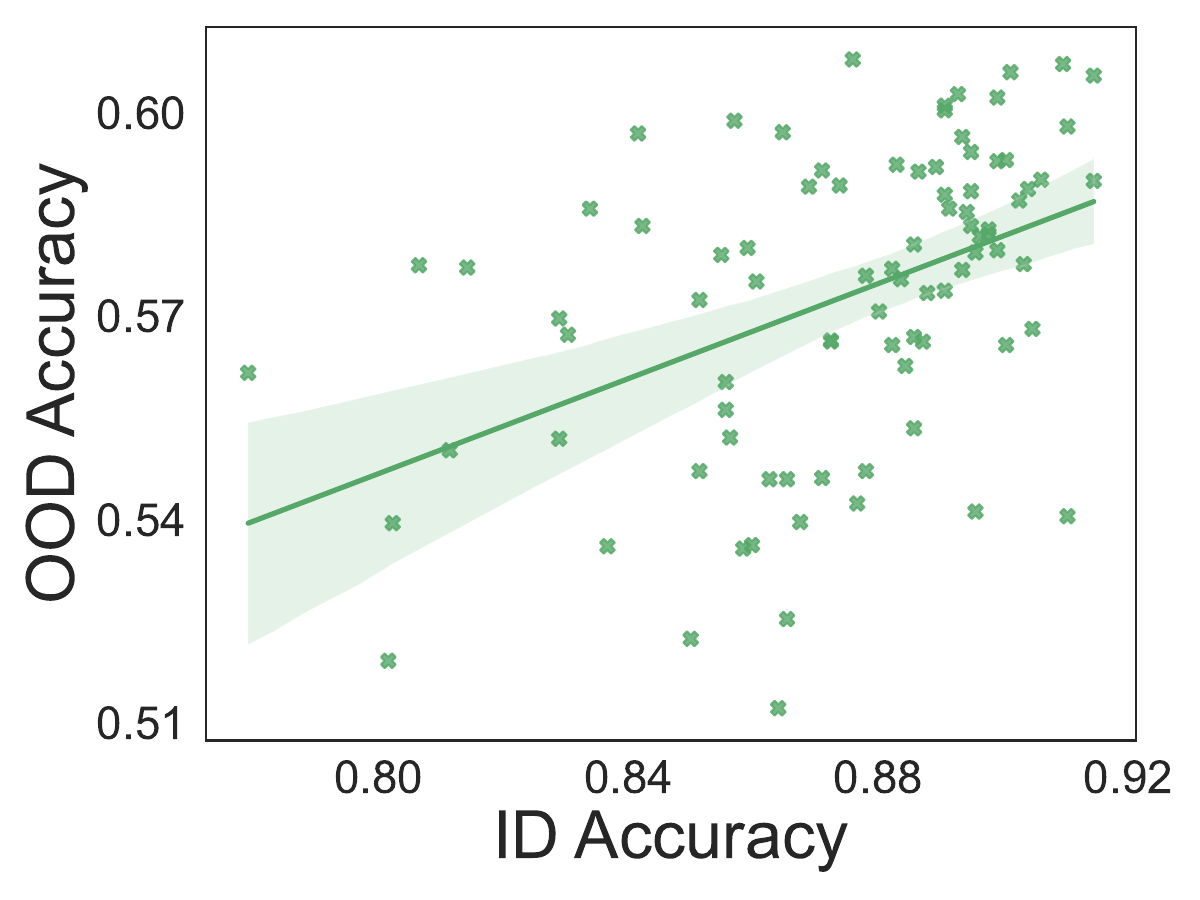}
    \caption{\footnotesize Noiseless Interpolation}
    \label{fig:fmow-exp-plot-noiseless}
\end{subfigure}
\begin{subfigure}[t]{0.32\linewidth}
    \centering
    \includegraphics[width=0.99\linewidth]{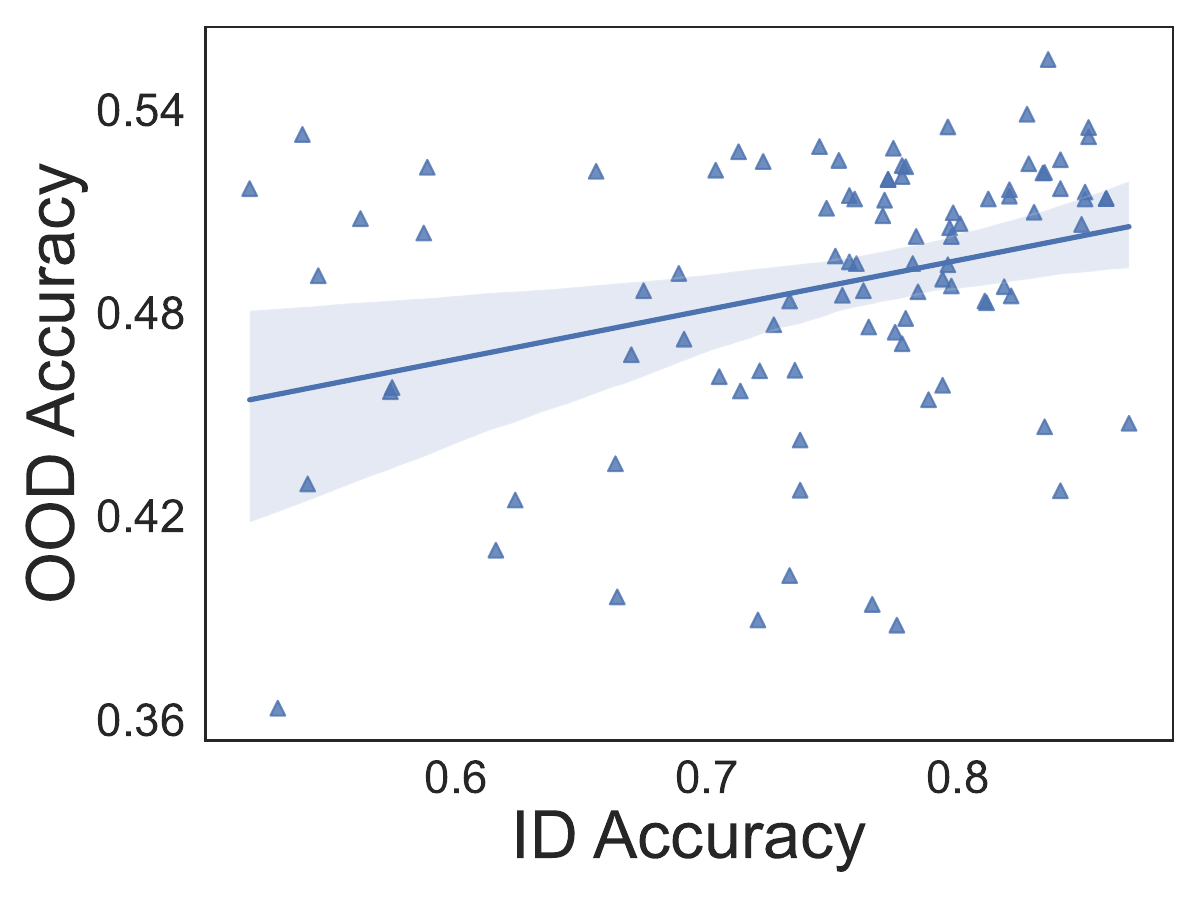}
    \caption{\footnotesize Noisy without interpolation}
    \label{fig:fmow-exp-plot-noisy-nointerp}
\end{subfigure}
\caption{Experiments on the \Gls{fmow} domain-correlated dataset with label noise. The noisy dataset (left) shows the~\awline phenomenon, while the noiseless dataset (center) shows the~\aline phenomenon. When the noisy dataset is not interpolated \eg due to early stopping (right), the~\aline phenomenon persists.}
\label{fig:fmow-exp-plot-main}
\end{figure}

The results in this section highlight that spurious correlations without label noise are insufficient to enforce~\awline. In~\Cref{app:exp-det}, we also show evidence (\Cref{fig:fmow-exp-plot}) that the presence of spurious correlations is necessary in these datasets.

\section{``Accuracy-on-the-wrong-line'' in theory: sufficient conditions}
\label{sec:theory}
In this section, we present our main theoretical results to isolate the main factors responsible for breaking the~\aline
phenomenon. First, we define the data distribution~\(\dist\)
in~\Cref{defn:disjoint-dist} and shift distribution~\(\shift\)
in~\Cref{defn:shift-dist}. The~\Gls{id} error is measured as the
expected error on \(\dist\) and the~\Gls{ood} error is the expected error on the shifted data i.e. on~\(x+\delta\) where \(x\sim\dist\) and \(\delta\sim\shift\). 

Intuitively, the main property of the distribution \(\dist\) is
that the ``signal'' and ``nuisance'' features are supported on
disjoint subspaces~(\(S_d\) and \(S_k\) respectively) and that the
shift \(\shift\) does not affect the signal features. Then,~\Cref{thm:main-thm}
lower bounds the difference between~\Gls{ood} and~\Gls{id} error, as a increasing function of the lower bound on nuisance sensitivity of the learned model. Further,
in~\Cref{thm:A1-informal} , we provide a theoretical example, where
the lower bound on nuisance sensitivity increases with label noise when the label
noise is interpolated. Taken
together,~\Cref{sec:main-reslt-thm,sec:theory-assump} proves that
under high dimensions of nuisance space, and high label noise
interpolation negatively impacts~\Gls{ood} error. ~\emph{We remark that in our theoretical results, we have avoided defining specific data distribution, distribution shifts, or learning algorithms in order to show a result on the general phenomenon and highlight the conditions that give rise to it. We leave it to future work to derive problem and algorithm specific statistical rates of~\Gls{ood} error.}
\subsection{Data distribution}\label{subsec:data-distribution} We model our data
distribution~\(\dist\) to have a few signal features and multiple
irrelevant~(or nuisance) features. This corresponds to real-world
settings where data is usually high dimensional but lies in a low
dimensional manifold. We further simplify the setting by restricting
this low-dimensional manifold to the linear subspace spanned by the
first \(d\in\bZ\) coordinate basis vectors. Formally,  for
\(d,k\in\bZ\) let \(S_d,S_k\) be any two disjoint subsets of
\(\bc{1\ldots d+k}\).  Without loss of generality, we assume them to
be contiguous i.e.  \(S_d=\bc{1,\ldots,d}\) and
\(S_k=\bc{d+1,\ldots,d+k}\).

\begin{defn}\label{defn:disjoint-dist} A distribution \(\dist\) on
\(\reals^{d+k}\times\bc{-1,1}\)~is called a \(\br{S_d,S_k}\)-disjoint
signal distribution with signal and nuisance support  \(S_d,S_k\)
respectively if there exists a linear separator \(w\in\reals^{d+k}\)
with its support exclusively on \(S_d\) and
\[\bE_{\br{x,y}\sim\dist}\bs{\ind{\sgn{\ip{w}{x}}\neq y}}=0.\]
\end{defn}

We define the shift distribution \(\shift\) as only impacting the nuisance features. This corresponds to widely held assumptions that distribution shifts do not affect the dependence of the label on the signal features. We define such shifts as \(S_d\)-oblivious shifts in~\Cref{defn:shift-dist}.

\begin{defn}\label{defn:shift-dist}
    A shift distribution \(\shift\) is called a \(S_d\)-oblivious shift distribution if the marginal distribution \(\shift_{S_d}\) on the support \(S_d\) is concentrated fully on \(\mathbf{0}_{d}\),~\ie~\[\shift_{S_d}\br{\mathbf{0}_{d}}=1.\]
\end{defn}

In short, both definitions assume that there are two orthogonal subspaces. For theoretical modelling, we consider that the signal subspace and the nuisance subspace are exactly disjoint in the standard coordinate basis. While this may not hold in the original data space in practice, this usually holds in the latent space as a few components in the latent space are sufficient to solve the problem at hand.  We regard those components as the signal space and the rest as the nuisance subspace. The assumption that the \(S_d\)-oblivious shift distribution has no mass on the signal space reflects a natural assumption about distribution shifts: they do not affect the causal factors in the data. 

\subsection{Properties of learned model}\label{subsec:assumpt}
The above definitions for~\Gls{id} and~\Gls{ood} data alone do not provide sufficient conditions to break the~\aline behavior. These definitions align with realistic settings, such as the sparsity of the true labeling function and distribution shifts orthogonal to the features in the signal space. Consequently, real-world experiments on the~\aline phenomenon already explore these conditions. Therefore, we next define three conditions on the learned model that we identify as sufficient to break this phenomenon. 

Let \(\what\in\reals^{d+k}\) be the learned sparse linear classifier
with support \(S\). Let \(\dist,\shift\) be any \((S_d,
S_k)\)-disjoint signal distribution and \(S_d\)-oblivious shift
distribution as defined in~\Cref{defn:disjoint-dist} for some
\(S_d,S_k\), and define \(\nu=\bE\bs{\shift}\) as the mean
of~\(\shift\). Then, we state the following conditions on
\(\what\).
\begin{itemize}[leftmargin=0.1in]
    \item[] \textbf{Condition~\ref{ass:1-sens}}:~``Bounded sensitivity" of \(\hat{w}\) on nuisance
    subspace assumes there exists \(M,\tau \geq 0\) s.t.
    \begin{equation}\label{ass:1-sens}\tag{C1}
    M\geq \max_{i \in S_k \cap S} |\hat{w}_i|\geq \min_{i \in S_k \cap S} |\hat{w}_i| \geq \tau.
    \end{equation}

    \item[] \textbf{Condition~\ref{ass:2-align}}:~``Negative Alignment" of \(\hat{w}\) with mean of shift
    \(\nu\) assumes there exists \(\gamma > 0\) s.t.
    \begin{equation}\label{ass:2-align}
    \gamma \leq -\frac{\sum_{i \in S_k \cap S} \hat{w}_i \nu_i}{\norm{\what_{S_k\cap S}}_1}.\tag{C2}
    \end{equation}

    \item[] \textbf{Condition~\ref{ass:3-margin}}:~``Small Margin" of \(\hat{w}\) assumes that for all
    \(x\) s.t. \(\ip{\what}{x}>0\), the following holds
    \begin{equation}\label{ass:3-margin}
        \ip{\hat{w}}{x}\leq \tau\gamma \abs{S_k \cap S}.\tag{C3}
    \end{equation}
\end{itemize}

\subsection{Main theoretical result}
\label{sec:main-reslt-thm}
Now, we are ready to state the main result. In~\Cref{thm:main-thm}, we
provide a lower bound on the~\Gls{ood} error corresponding to a fixed
\(x\) where the randomness is over the sampling of the shift \(\delta\sim\shift\). 

\begin{restatable}{thm}{mainthm}\label{thm:main-thm} For any \(S_d,S_k\) let
    \(\mathcal{D}\) be a \((S_d, S_k)\)-disjoint signal distribution,
    and \(\shift\) be a \(S_d\)-oblivious shift distribution where
    each coordinate is an independent subgaussian with parameter
    \(\sigma\). 
    
    Then, for any \(x\in\mathrm{dom}\br{\dist}\) and \(\what \in
    \mathbb{R}^{d+k}\) with support \(S\) such that \(\what\)
    satisfies Conditions~\ref{ass:1-sens},~\ref{ass:2-align}, 
    and~\ref{ass:3-margin},  we have $\prob_{\delta}\br{\ip{\hat{w}}{x + \delta} \leq 0} \geq 1- e^{-\Gamma}$ where 
    \begin{equation}
    \label{eq:thm1-err-bnd}
    \Gamma =\frac{\abs{S_k \cap S}\br{\tau\gamma - \nicefrac{\cC}{\abs{S_k\cap S}}}^2 }{2 \sigma^2 M^2}. 
    \end{equation}
   for all \(x\in\mathrm{dom}\br{\dist}\) where \(\ip{\what}{x}\geq
   0\) and \(\cC=\max_{\ip{\what}{x}\geq 0}\ip{\what}{x}\).
\end{restatable}

\looseness=-1\Cref{thm:main-thm} proves that for all~(positively) correctly classified points, the probability of misclassification under the~\Gls{ood} perturbation $\delta$ increases with $\Gamma$. In particular, \(\Gamma\) scales with the nuisance sensitivity \(\tau\) and nuisance density \(\abs{S_k\cup S}\). Our experiments later show that increasing data size, which leads to lower~\Gls{id} error, in fact increases nuisance density, which as~\Cref{thm:main-thm} suggest, leads to larger~\Gls{ood} error.

The proof of the theorem is based on applying a Chernoff-style bound on a weighted combination of \(k\) sub-gaussian random variables; see \Cref{app:proof} for the full proof. \Cref{thm:main-thm} captures a broad class of distribution shifts, including bounded and normal distributions. The results can also be extended to other shifts with bounded moments but we omit them here as they add more mathematical complexity without additional insights. In particular, under the conditions of~\Cref{thm:main-thm}, for some \(\sigma>0,\nu\in\reals^{d+k}\) where \(\nu\) satisfies A2, consider either:
    \begin{itemize}
        \item \textbf{Gaussian Shifts:} Each \(\delta_i\) for \(i \in S_k \cap S\) is independently distributed as \(\cN(\nu_i, \sigma^2)\), or
        \item \textbf{Bounded Shifts:} Each \(\delta_i\) for \(i \in S_k \cap S\) satisfies \(\abs{\delta_i - \nu_i} \leq \sqrt{3}\sigma\).
    \end{itemize}
    Then, the probability that \(\what\) misclassifies \(x\)  which satisfies~\eqref{ass:3-margin}  under the shift \(\delta\) is bounded by the same expression as in~\Cref{eq:thm1-err-bnd}.

\subsection{Understanding and relaxing conditions}
\label{sec:theory-assump}

We next argue why these conditions are merely abstractions of phenomena already observed in practice, as opposed to strong synthetic constraints absent in applications. In addition, we also show how some of these conditions can be significantly relaxed.

\textbf{Relaxing Condition~\eqref{ass:2-align} and~\eqref{ass:3-margin}.} Our setting is not restricted to only discussing samples from one class, balanced classes, or cases where all data points have small margins. In this section, we relax Condition~\eqref{ass:2-align} and \eqref{ass:3-margin} to allow for imbalanced classes and for some data points to have large margins. Condition~\ref{ass:3-margin} requires that for all data points that are positively classified, the margin of classification is bounded from above. Note that this is not a limitation of our result. A simple corollary~(\Cref{corr:imbalanced}) states the proportion of \(\dist\) for which this holds directly affects the proportion for which the~\Gls{ood} performance is poor. We use the notation \(\rho\) in~\Cref{ass:4-partalign} to characterise the fraction of the dataset classified positively (or negatively, whichever yields a higher \(\rho\)) by \(\what\) with a margin that is less than half of the maximum allowed margin.

Condition~\eqref{ass:2-align} requires that the distribution shift should not orthogonal to \(\what\). This is not a strict requirement,
as exact orthogonality of \(\nu\) with \(\hat{w}\) is a very unlikely
setting, and even slight misalignment will suffice for our result. Here, we show that if the shift distribution is a mixture of multiple components with a combination of positive and negative alignments, our result extends to that setting. Consider a new shift distribution \(\shift^\prime\), which
is a mixture of two shift distributions \(\shift_1\) and
\(\shift_{-1}\) with mixture coefficients \(c_1\) and \(c_{-1}\),
respectively. Now note that at least one of the two-component shift
distributions will likely satisfy Condition~\eqref{ass:2-align} with
\(\what\) or \(-\what\). Assume, \(\shift_1\) satisfies
condition~\eqref{ass:2-align} with \(\gamma_1\), and \(\shift_{-1}\)
satisfies the same condition by replacing \(\what\) with \(-\what\)
for \(\gamma_{-1}\). 

As shown in~\Cref{corr:imbalanced}, when the distribution becomes more class-imbalanced and a large fraction of data points have small margins and at least one of the distributions has a large negative alignment \(\gamma\), the parameter \(\rho\) increases, thereby increasing the probability of misclassification.

\begin{corollary}
\label{corr:imbalanced} Define \(S_d,S_k\), and \(\dist\) as in~\Cref{thm:main-thm} and \(\shift^\prime\) as described above. Consider any \(\what \in
\mathbb{R}^{d+k}\) with support \(S\) such that \(\what\) satisfies
Conditions~\ref{ass:1-sens} and \ref{ass:2-align}. Define
\begin{equation}\label{ass:4-partalign}\rho=\max_{\hat{y}\in\bc{-1,1}}\prob_{x\sim\dist}\bs{\bI\bc{\hat{y}\ip{\what}{x}\geq 0}\cdot \bI\bc{\hat{y}\ip{\what}{x}\leq \frac{\tau\gamma_{\hat{y}} \abs{S_k \cap S}}{2} }}.\tag{C4}\end{equation}
Then, we have
    \[
    \prob_{x,\delta}\br{\ip{\hat{w}}{x + \delta} \neq \ip{\what}{x}} \geq \rho_c \sum_{i\in \{-1, 1\}}c_{i}\br{1 - \exp\bc{-\frac{\abs{S_k \cap S}\tau^2\gamma_{i}^2}{8 \sigma^2 M^2}}}.
    \]
\end{corollary}

The above result shows how the~\Gls{ood} error adaptively depends on various properties of the learned classifier and shift distribution. It shows that the~\Gls{ood} error increases with the increase in the density of \(\what\) in the nuisance subspace,~\ie~\(\abs{S_k\cap S}\), as well as the ratio of the minimum and maximum spurious sensitivity \(\frac{\tau}{M}\), from Condition~\ref{ass:1-sens}. Second, the increase in the negative alignment \(\gamma\) increases the~\Gls{ood} error and it depends on which class has the worse parameters. Finally, we note that \(1-\rho\) upper bounds~\Gls{id} error. Therefore, while a larger \(\rho\) leads to a larger lower bound on the~\Gls{ood} error, it leads to a smaller upper bound on the~\Gls{id} error--- a reflection of the \awline behaviour.

\textbf{Understanding Condition~\eqref{ass:1-sens}.}
 Condition~\eqref{ass:1-sens} describes the condition that the learned
classifier has moderately large values in its support on the nuisance
subspace.  We provide an example in~\Cref{thm:A1-informal} to show that this naturally occurs when interpolating label noise. Consider a min-$\ell_2$-interpolator that solves the following optimization problem given a dataset $(X, Y)\in \reals^{n\times d}\times \reals^n$, 
    \[\min_{w\in \reals^d}\norm{w}_2\quad \mathrm{s.t.}~Xw = Y.\]

Consider a simple linear model with a single signal feature $Y= \xi \odot \ip{X}{w^\star}$ where $w^\star = \br{1, 0, \ldots, 0}$. Here, $X$ is a $d$-dimensional dataset of size \(n\) with signal feature $X_1\sim \cN(\upsilon\mathbf{1}_n, I_n)$ and the remaining \(d-1\) nuisance features $X_{2:d}\sim \cN(0, I_n)$.\footnote{Here, $\mathbf{1}_n$ denotes an all-ones vector of length $n$.} The label noise $\xi\in \reals^{n}$ follows the distribution $\pi$ and $\odot$ denotes the Hadamard product. When~\(\xi\) is a Bernoulli random variable on the set $\bc{-1, 1}$, it captures the setting of uniformly random label flip.

\begin{proposition}[Informal]\label{thm:A1-informal}
    If for $d = \Omega\br{\log n}$, some constant $C > 0$, and $\beta\in (0, 1)$, the noise distribution $\pi$ satisfies $\prob\bs{\norm{X_{2:d}^{+} \xi}_2> C}\geq 1-\beta$, where $A^+$ is the pseudo-inverse of $A$, then with probability at least $0.9-\beta$, the min-$\ell_2$-interpolator $\hat{w}$ on the noisy dataset $(X, Y)$ satisfies \[\norm{\hat{w}_{2:d}}_\infty =\Omega\br{C}.\]     
\end{proposition}

While~\Cref{thm:A1-informal} only considers multiplicative label noise for simplicity of the analysis, similar results also hold for additive label noise models. \Cref{thm:A1-informal} also captures the properties of label noise that are sufficient to increase the sensitivity of the nuisance features. It suggests that a label noise distribution, whose (noisy) labels are nearly orthogonal to the nuisance subspace (indicated by large \(C\) and small \(\beta\)), induces small nuisance sensitivity. For example, a noiseless setting is equivalent to the noise distribution where $\prob_{\xi\sim \mu}\bs{\xi = \mathbf{1}_n} = 1$. Then, standard concentration bounds on \(X_{2:d}\) imply that \(C\) must be small while \(\beta\) must be large, leading to a vacuous bound. Therefore,~\Cref{thm:A1-informal} implies that the lower bound on nuisance sensitivity is much smaller in the noiseless setting.

Next, we verify these properties on a simple learning problem in~\Cref{sec:linear}. Our experiments show that increasing dataset size can naturally increase \(\abs{S_k \cap S}\) and \(\tau\), thereby leading to a larger \(\Gls{ood}\) error. Conversely, increasing the dataset size will also lead to a lower~\Gls{id} error. Together, this demonstrates an inverse correlation between~\Gls{id} and~\Gls{ood} error, displaying~\awline.

\section{Experimental ablation of sufficient conditions in linear setting}
\label{sec:linear}
\begin{figure}[t]
    \centering
    \begin{subfigure}[b]{0.06\linewidth}
    \includegraphics[width=1.0\linewidth]{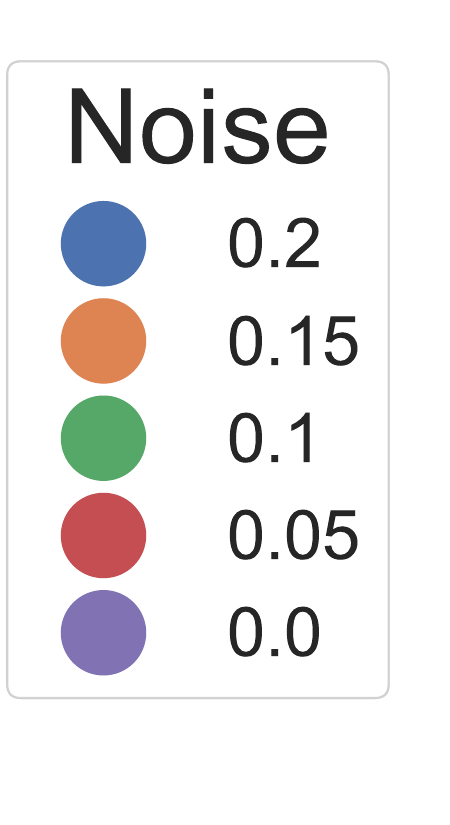}
    \end{subfigure}
    \begin{subfigure}[t]{0.32\linewidth}
    \begin{subfigure}[t]{0.49\linewidth}
    \includegraphics[width=1.\linewidth]{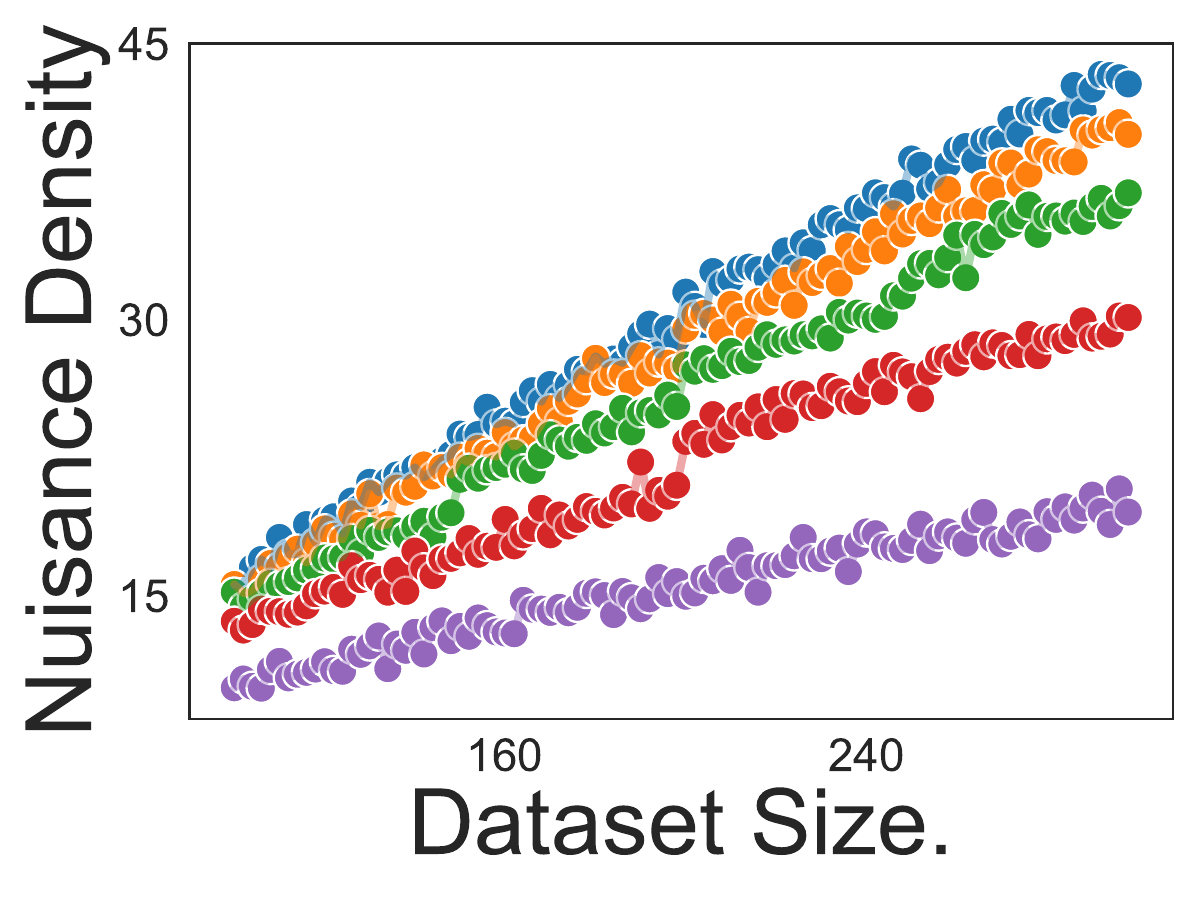}
    \end{subfigure}  
    \begin{subfigure}[t]{0.49\linewidth}  
    \includegraphics[width=1.\linewidth]{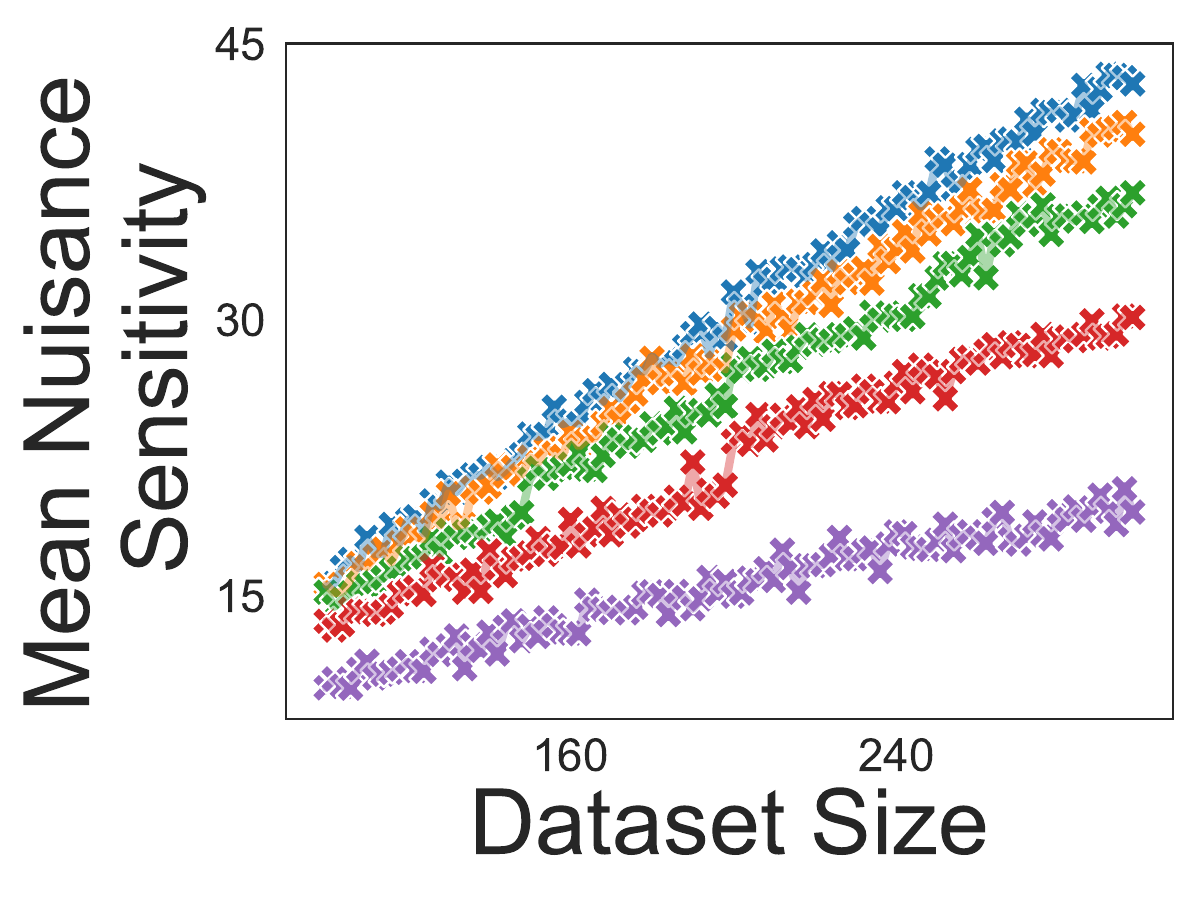}
    \end{subfigure}
    \caption{Nuisance Sensitivity}
    \label{fig:var-noise-sens}
    \end{subfigure}\qquad
     \begin{subfigure}[t]{0.32\linewidth}
    \begin{subfigure}[t]{0.49\linewidth}  
    \includegraphics[width=1.\linewidth]{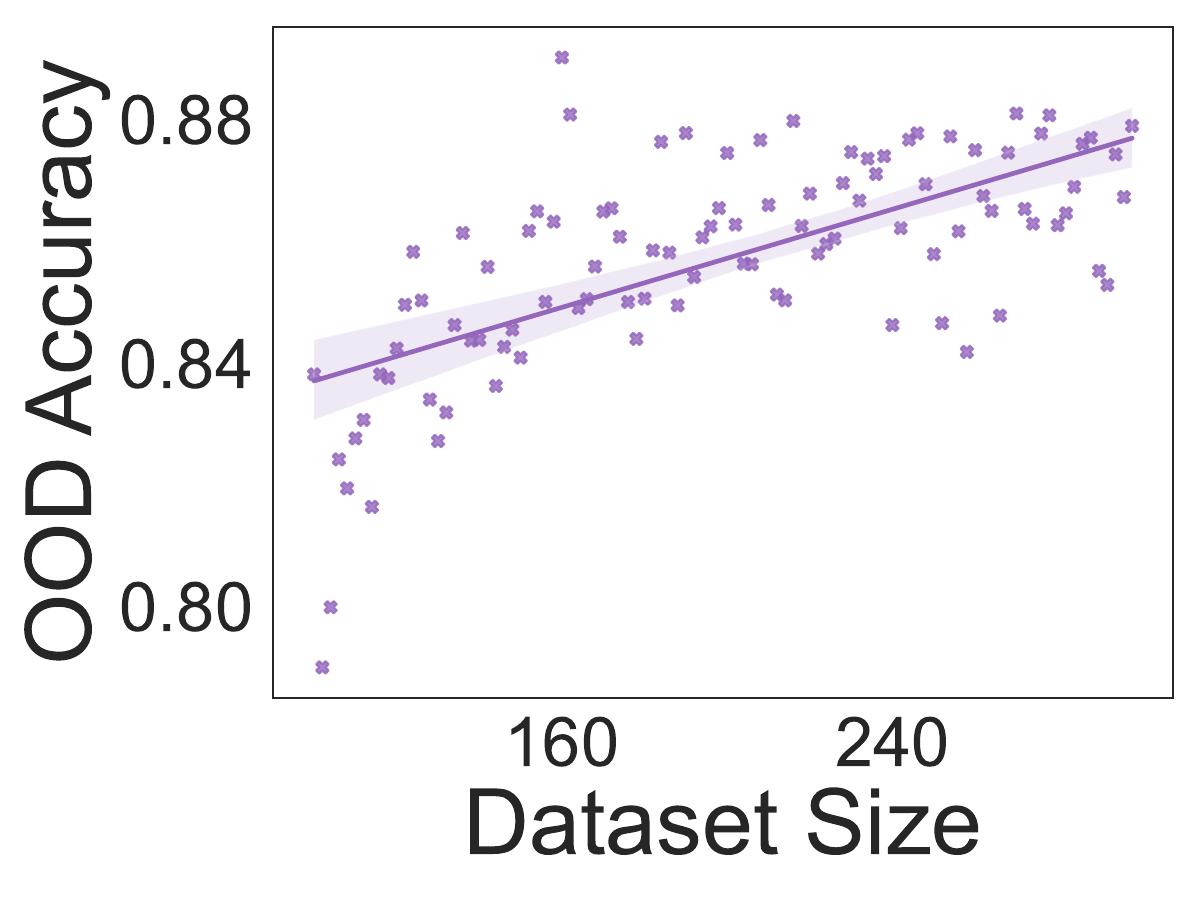}
    \end{subfigure}
    \begin{subfigure}[t]{0.49\linewidth}  
    \includegraphics[width=1.\linewidth]{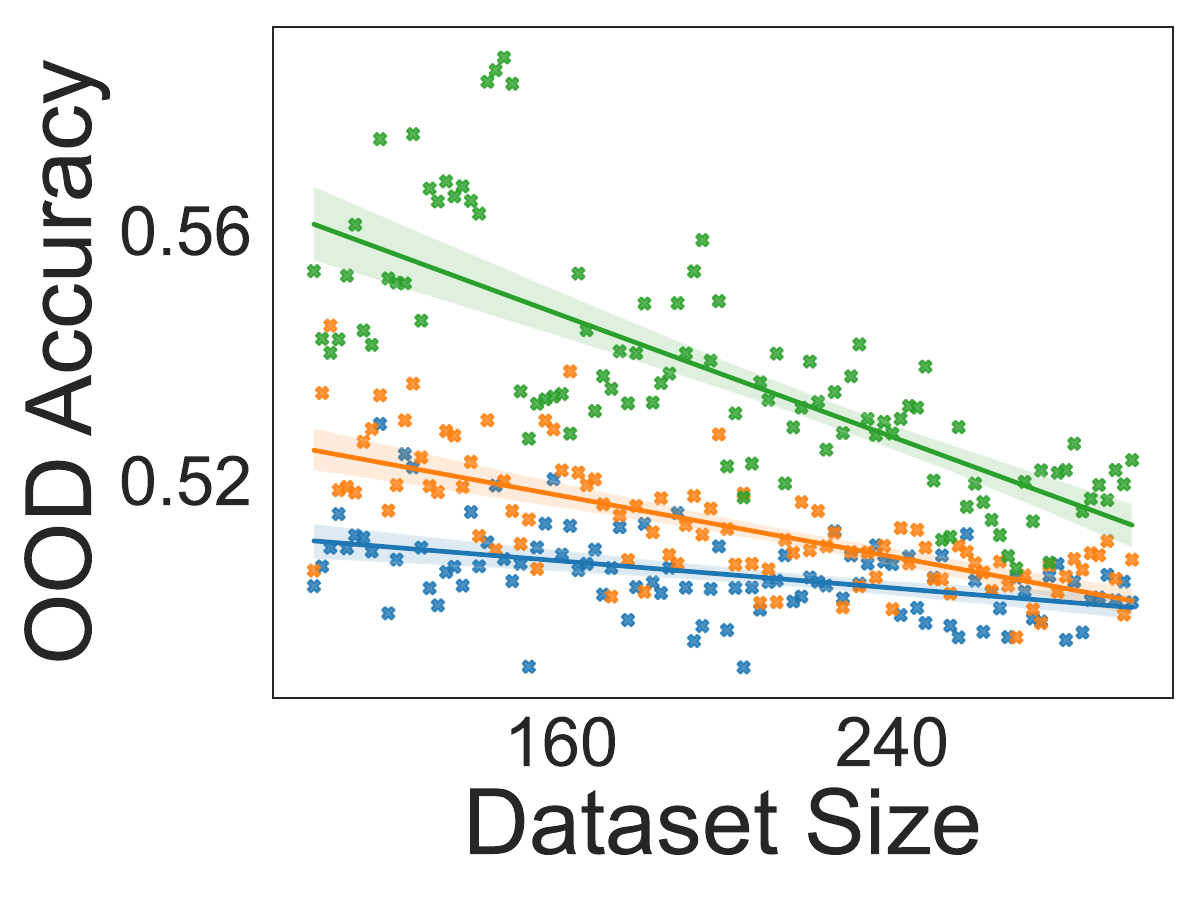}
    \end{subfigure}
    \caption{OOD Accuracy}
    \label{fig:var-noise-ood}
    \end{subfigure}\qquad
    \begin{subfigure}[t]{0.16\linewidth}  
    \includegraphics[width=1.\linewidth]{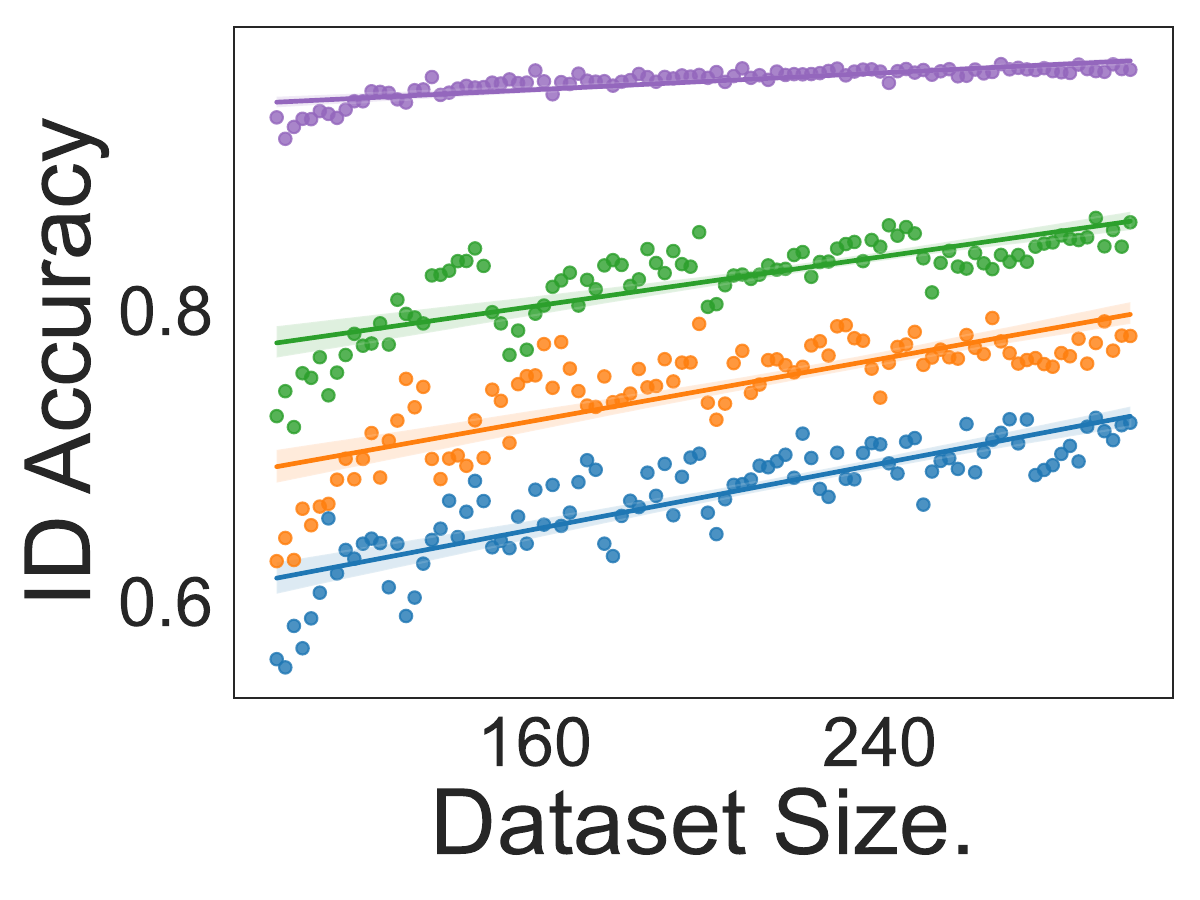}
    \caption{ID Accuracy}
    \label{fig:var-noise-id}
    \end{subfigure}
    \caption{\Cref{fig:var-noise-sens} shows as the amount of label noise increases, nuisance sensitivity as well as the nuisance density increases faster with larger dataset sizes. This leads to worse~\Gls{ood} accuracy as shown in~\Cref{fig:var-noise-ood}. However,~\Gls{id} accuracy still increases with dataset size as shown in~\Cref{fig:var-noise-id}.
    \label{fig:vary_noise}
    }
    
    \end{figure}
    
In this section, we conduct experimental simulations to corroborate
our theory by synthetically varying
conditions~\eqref{ass:1-sens},~\eqref{ass:2-align},
and~\eqref{ass:4-partalign} as well as label noise rate and dataset sizes.  The data distribution is \(300\)-dimensional with one signal feature and the remain nuisance features, a sparse setting often considered in the literature; See~\Cref{app:synthetic-exp-setting} for a detailed discussion of the data distribution. The default label noise rate is \(0.2\) unless otherwise mentioned and the default dataset size is \(300\) unless otherwise mentioned. We train a  \(\ell_1\)-penalised logistic regression classifier with coefficient \(0.1\) on varying dataset sizes. In short, our experiments show that all three conditions hold for this learned model in the presence of label noise and corroborates our theory regarding how these problem parameters affect the~\Gls{ood} and~\aline phenomenon.

    \begin{figure}\centering
    \begin{subfigure}[t]{0.19\linewidth}
    \includegraphics[width=1.\linewidth]{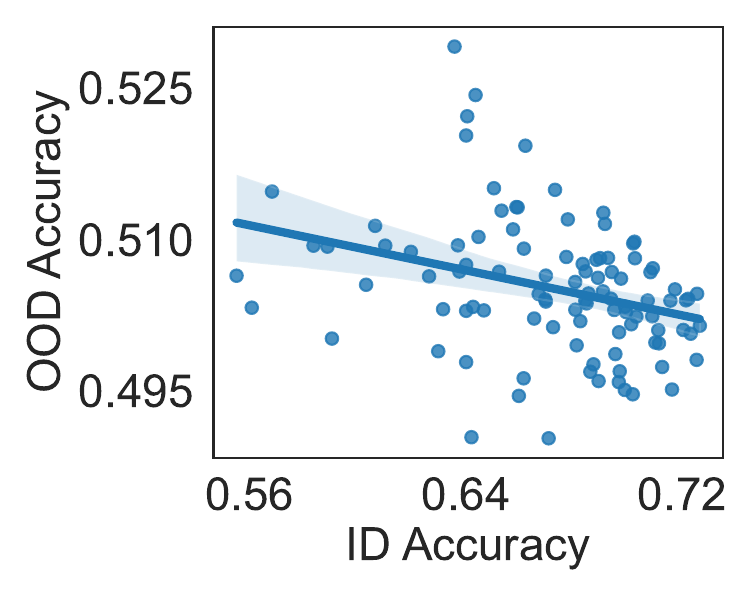}
    \caption*{\(\eta=0.2\)}
    \end{subfigure}
    \begin{subfigure}[t]{0.19\linewidth}
    \includegraphics[width=1.\linewidth]{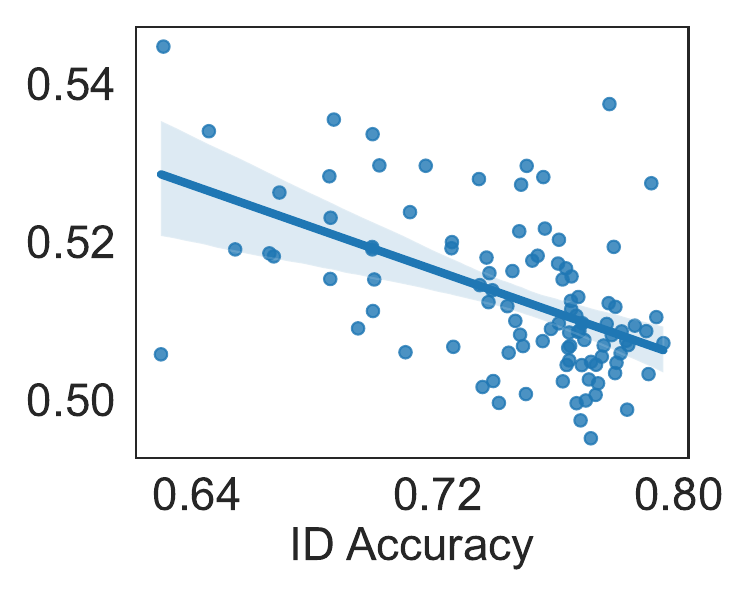}
    \caption*{\(\eta=0.15\)}
    \end{subfigure}
    \begin{subfigure}[t]{0.19\linewidth}
    \includegraphics[width=1.\linewidth]{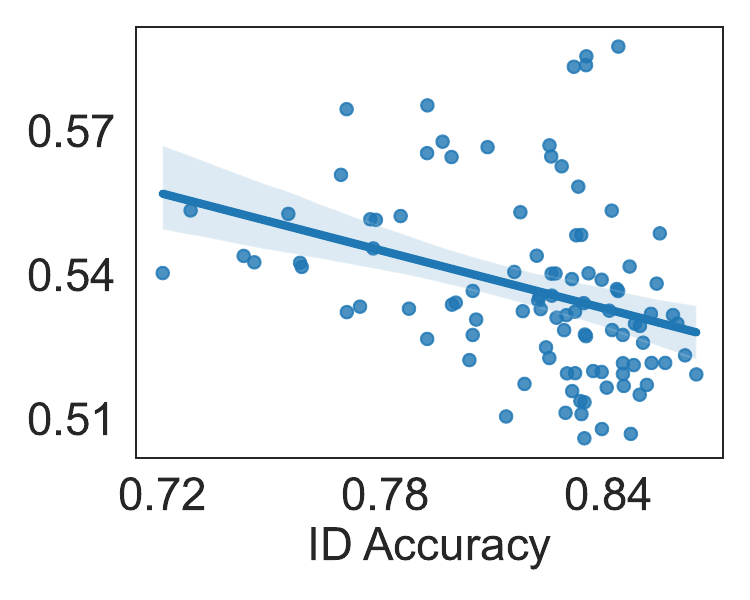}
    \caption*{\(\eta=0.1\)}
    \end{subfigure}
    \begin{subfigure}[t]{0.19\linewidth}
    \includegraphics[width=1.\linewidth]{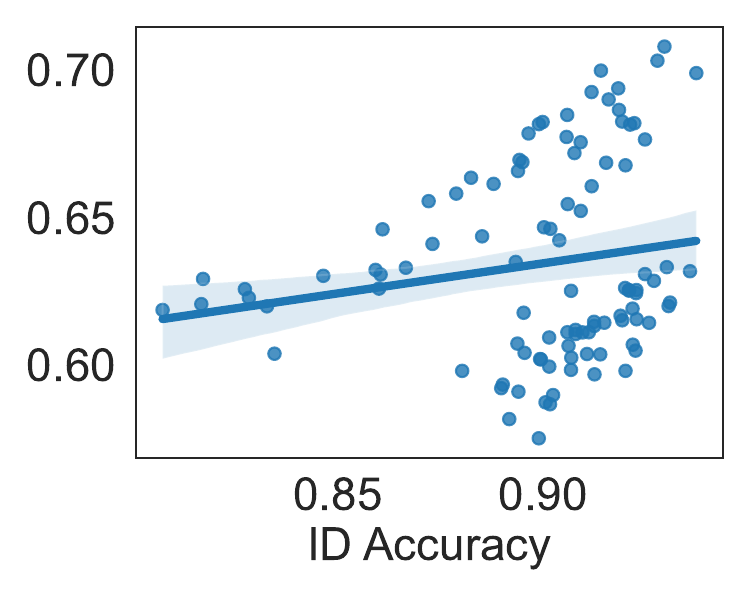}
    \caption*{\(\eta=0.05\)}
    \end{subfigure}
    \begin{subfigure}[t]{0.19\linewidth}
    \includegraphics[width=1.\linewidth]{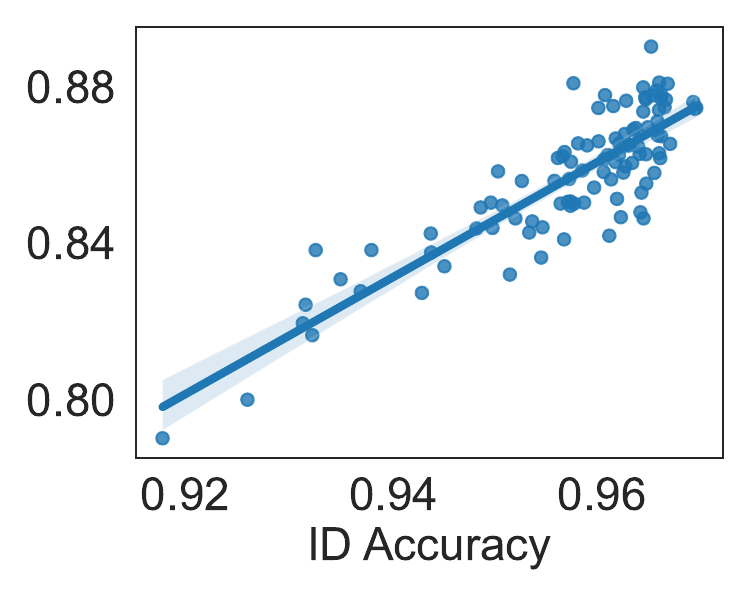}
    \caption*{\(\eta=0\)}
    \end{subfigure}
    \caption{\small \aline behaviour degrades increasing with increasing amount of label noise.}
    \label{fig:id-ood-vary-noise}
\end{figure}

\textbf{\eqref{ass:1-sens}: Spurious Sensitivity of the Learned Model.} We begin by examining the sensitivity of the learned model \(\what\) in the nuisance subspace. Condition~\eqref{ass:1-sens} states that the non-zero components are bounded both from above and below. While regularisation naturally imposes the upper bound, the lower bound is less common. A key contribution of this work is the demonstration that this occurs under \emph{noisy interpolation},~\ie, models that achieve zero training error in the presence of label noise. Intuitively, when some labels in the training dataset are noisy, the signal subspace cannot be used to ``memorise'' them. Consequently, covariates in the nuisance subspace are necessary to memorise these labels, thereby increasing the magnitude of these covariates. As the amount of label noise increases, more covariates in the nuisance subspace exhibit this behaviour. We corroborate this intuition using \(\ell_1\)-penalised logistic regression in experimental simulations, as shown in~\Cref{fig:var-noise-sens}.

\Cref{fig:var-noise-sens} illustrates that, with higher levels of label noise, the nuisance density and mean nuisance sensitivity increase more rapidly as the dataset size grows. ~\Cref{thm:main-thm} predicts that an increase in nuisance sensitivity leads to poorer~\Gls{ood} accuracy, which is confirmed in~\Cref{fig:var-noise-ood}~(center). However,~\Cref{fig:var-noise-ood}~(left) demonstrates that this behaviour is not observed in the absence of label noise; \Gls{ood} accuracy still improves with an increasing dataset size. ~\Cref{fig:var-noise-id} reveals that~\Gls{id} accuracy increases with larger datasets, thereby creating a distinction between the behaviour of~\Gls{id} and~\Gls{ood} accuracy in the presence of label noise. This distinction underpins the central observation of our paper, as illustrated in~\Cref{fig:id-ood-vary-noise}. For \(\eta=0\) (no label noise),~\Gls{id} and~\Gls{ood} accuracy are linearly correlated as noted in several prior studies~\citep{miller2021accuracy}. Conversely, as \(\eta\) increases, the two accuracies become (nearly) inversely correlated, resulting in the~\awline behaviour.

\begin{figure}[t]\centering
    \begin{subfigure}[t]{0.25\linewidth}
    \includegraphics[width=1.\linewidth]{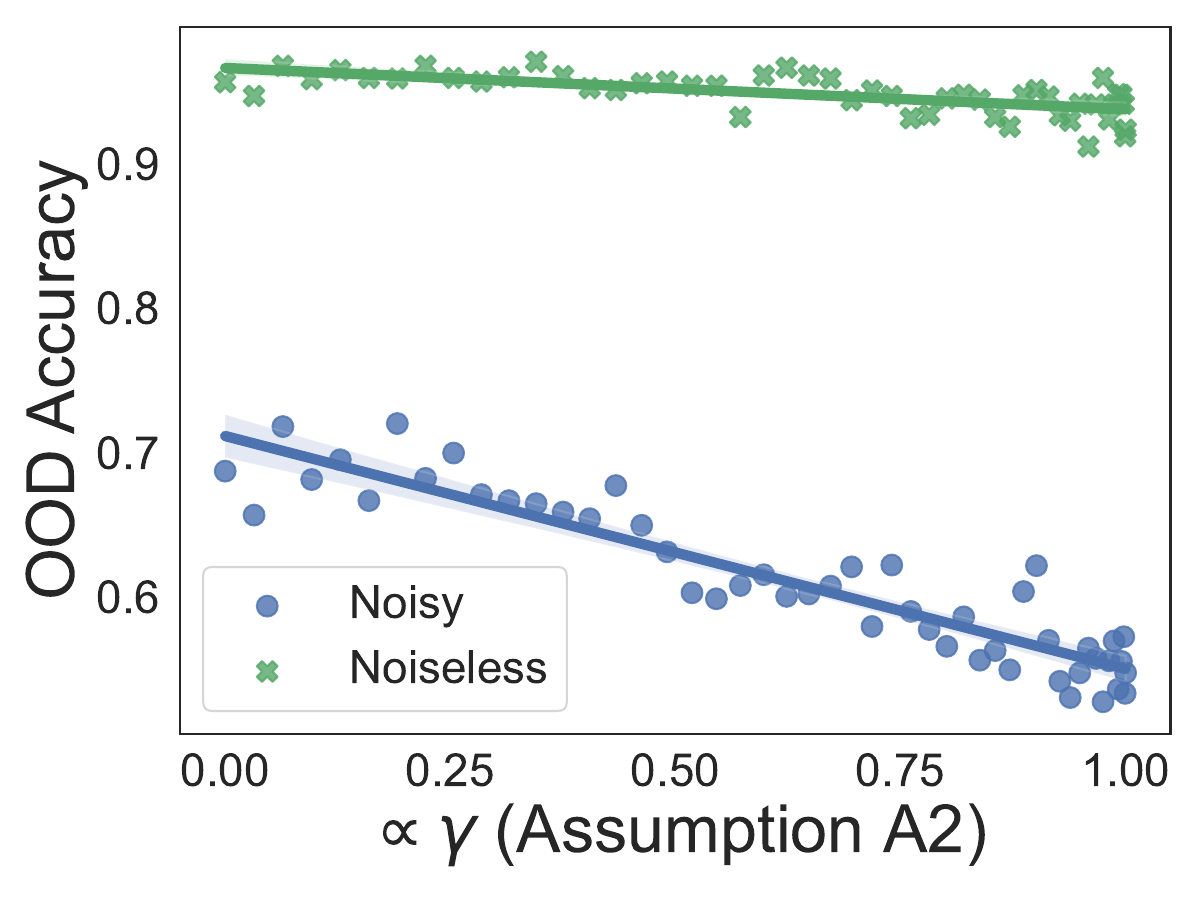}
    \caption{OOD Accuracy vs \(\gamma\)}
    \label{fig:ood-acc-noisy-gamma}
    \end{subfigure}\qquad
    \centering
    \begin{subfigure}[t]{0.25\linewidth}\centering
    \includegraphics[width=1.\linewidth]{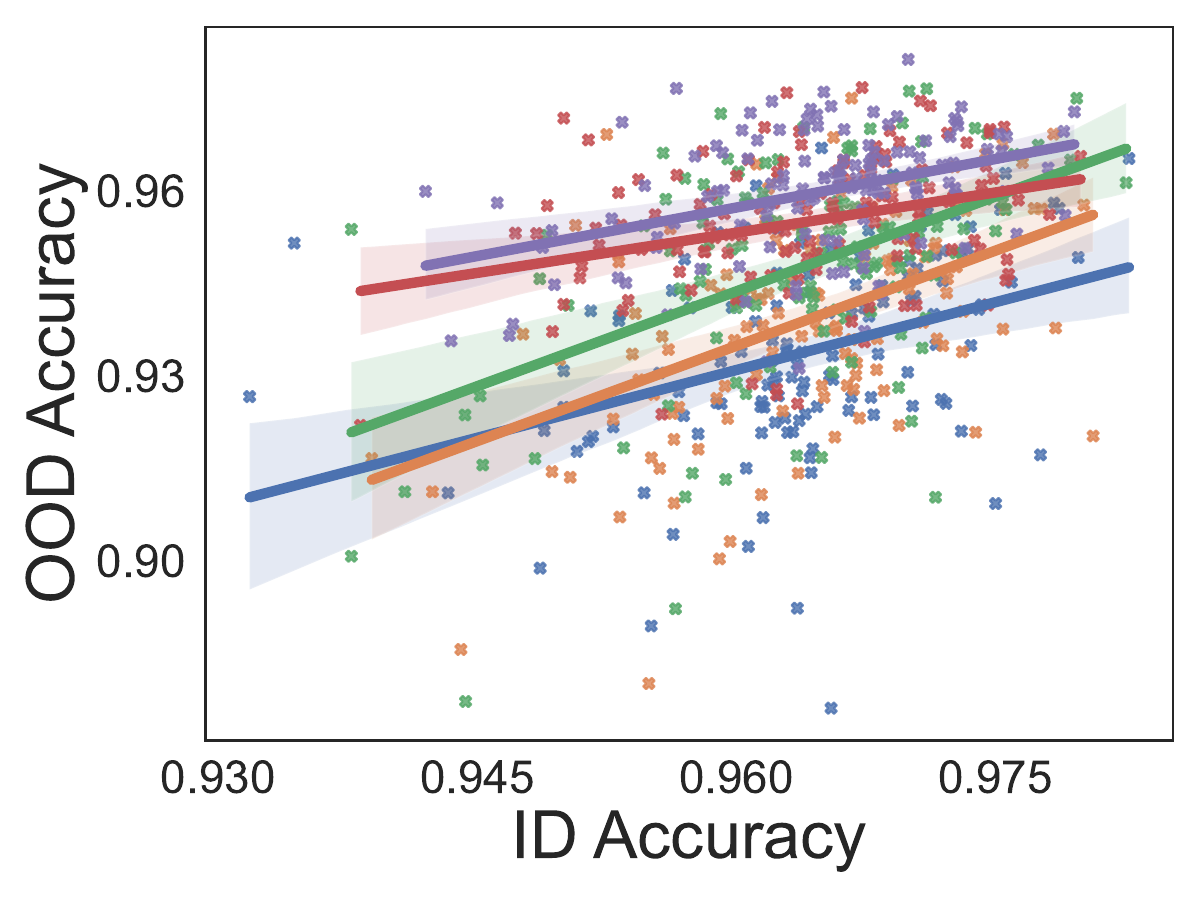}
    \caption{Noiseless}
    \label{fig:id-ood-gamma-noiseless}
    \end{subfigure}
    \begin{subfigure}[t]{0.25\linewidth}\centering
    \includegraphics[width=1.\linewidth]{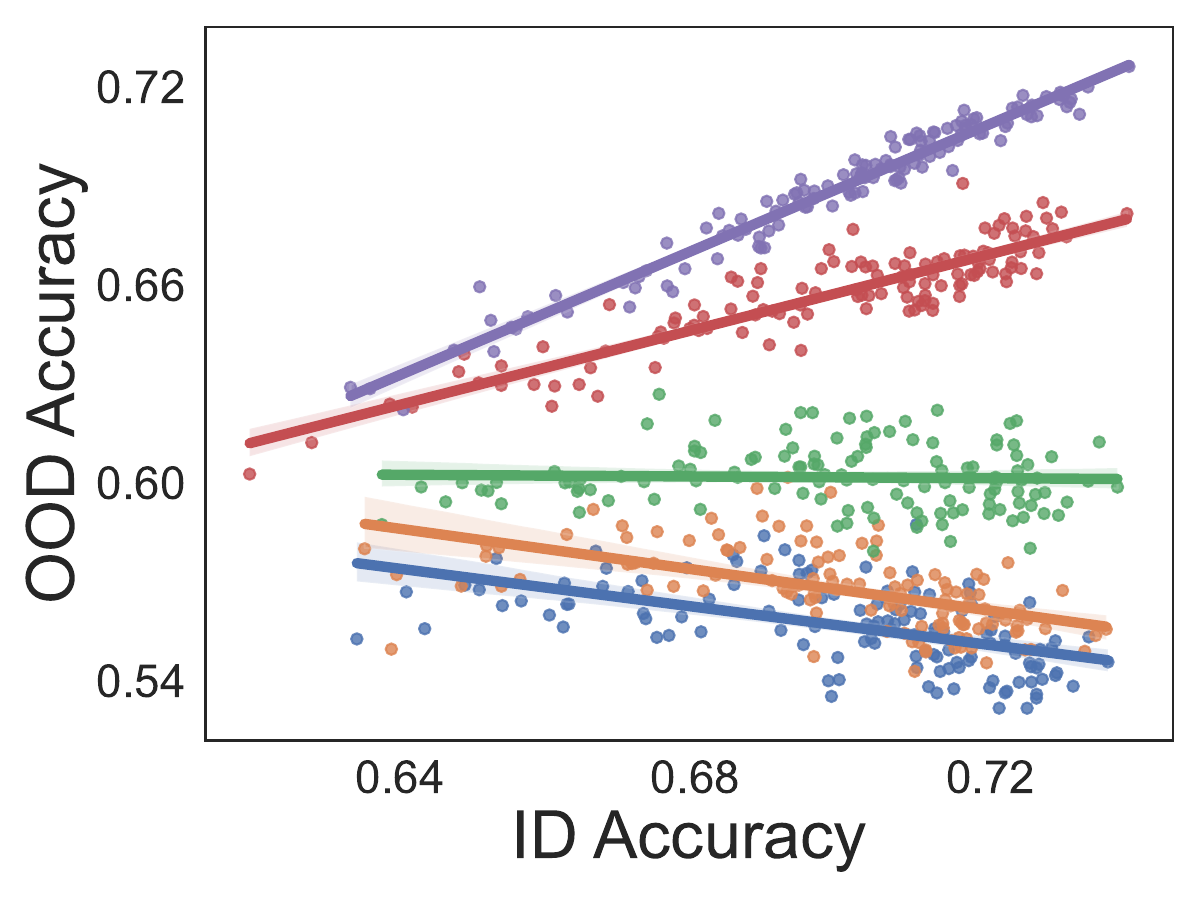}
    \caption{Noisy}
    \label{fig:id-ood-gamma-noisy}
    \end{subfigure}
    \begin{subfigure}[b]{0.14\linewidth}\centering
    \includegraphics[width=1.\linewidth]{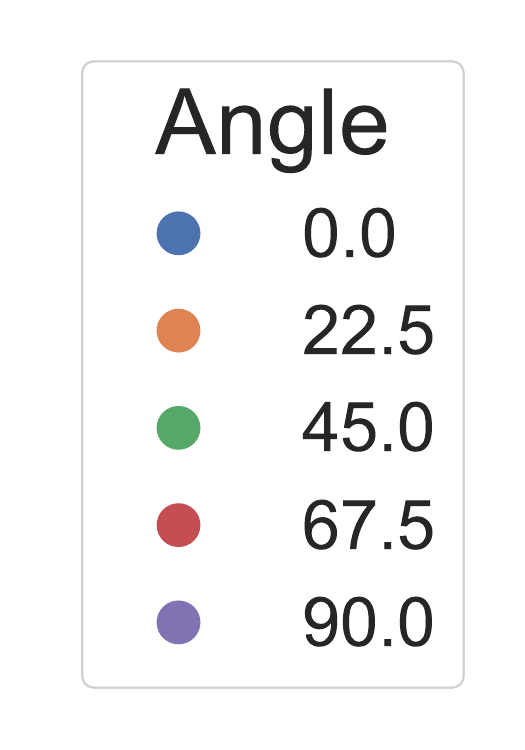}
    \end{subfigure}
    \caption{\small ~\Cref{fig:ood-acc-noisy-gamma} shows that as \(\gamma\) increases, the~\Gls{ood} accuracy decreases for the noisy setting but remains relatively unchanged for the noiseless setting. Figure~\ref{fig:id-ood-gamma-noiseless} shows that a large or small angle~(\(\cos\br{\mathrm{Angle}}=\gamma\)) does not have any impact on breaking the~\aline phenomenon. On the other hand, when angle is close to \(\ang{0}\)~\ie \(\what\) and \(\nu\) in~\ref{ass:2-align} are fully mis-aligned, the accuracy-on-the-wrong line behaviour is the strongest. The phenomenon slowly transforms to~\aline as the angle approaches \(\ang{90}\) degrees.}
\end{figure}

\textbf{\eqref{ass:2-align}: Mis-alignment of Learned Model with Shift distribution.} The next condition on \(\what\) requires that the mean \(\nu\) of the shift distribution \(\shift\) is misaligned. Mathematically, this requires that the dot product \(\langle \what, \nu \rangle \) is negative. Intuitively, this ensures that the term \(\langle \what, \delta \rangle\), where \(\delta \sim \shift\), is sufficiently negative to flip the decision of the classifier. In~\Cref{thm:main-thm}, the alignment is measured using the quantity \(\gamma\). In~\Cref{fig:ood-acc-noisy-gamma}, we synthetically vary \(\gamma\)~(represented by the angle between \(\nu\) and \(-\what\)) by controlling the projection of \(\nu\) on \(-\what\) and evaluate the~\Gls{ood} accuracy in the presence and absence of noise, respectively. The simulation shows that with increasing alignment between \(\nu\) and \(-\what\),~\Gls{ood} accuracy decreases sharply for the noisy setting but remains relatively stable for the noiseless setting.

In~\Cref{fig:id-ood-gamma-noiseless,fig:id-ood-gamma-noisy}, we show how the~\aline behaviour is affected by changing the alignment \(\gamma\). For the noiseless setting, irrespective of \(\gamma\),~\Gls{ood} accuracy remains positively correlated with~\Gls{id} accuracy. However, for the noisy setting, when the alignment is high (i.e., the angle is less than \(45^\circ\)), we observe that~\Gls{ood} accuracy is inversely correlated with~\Gls{id} accuracy, but they become more positively correlated as the angle increases. This highlights the necessity of our second condition, which requires a misalignment between the shift and the learned parameters. This also highlights that perfect misalignment is not strictly necessary for breaking~\aline.

\textbf{\eqref{ass:4-partalign}: Significant fraction of points have
low margin.} The final property of the learned classifier necessary
for the~\awline phenomenon is that a significant portion of the
distribution, correctly classified by \(\what\), has a small margin.
This is captured in~\ref{ass:4-partalign}. \Cref{corr:imbalanced}
suggests that the probability mass of points under distribution
\(\dist\) whose margin \(\langle \what, x \rangle\) is less than
\(\gamma \tau |S_k \cap S| \approx \gamma \|\what_{S_k \cap S}\|_1\)
is roughly equal to the probability mass of points under \(\dist\)
that are vulnerable to misclassification under the distribution shift.
Practically, a point classified correctly with a large margin is
likely robust to distribution shifts. However, it is typically the
case that not all points are classified with an equally large margin,
and some points are closer to the margin than others. %
We highlight that as
long as this is true,~\awline will continue to hold.

To validate this experimentally, we consider the distribution of~\Gls{id} margin \(\langle \what, x \rangle\) for \(x \sim \dist\) and plot its CDF in~\Cref{fig:margin-dist} (blue line). Then, we measure the term \(\gamma \tau |S_k \cap S|\) and plot it as a vertical red dashed line. The intersection of this red line with the CDF (blue) represents the probability mass of points under \(\dist\) whose margin is \emph{sufficiently small} to be vulnerable to the distribution shift. We plot the empirical~\Gls{ood} error for this model using the horizontal green dashed line and repeat this experiment for multiple dataset sizes, each represented in one box in~\Cref{fig:margin-dist}. Our simulations clearly show that the theoretically predicted quantity closely approximates the true~\Gls{ood} error.

\begin{figure}[!htb]
    \centering
    \includegraphics[width=1\linewidth]{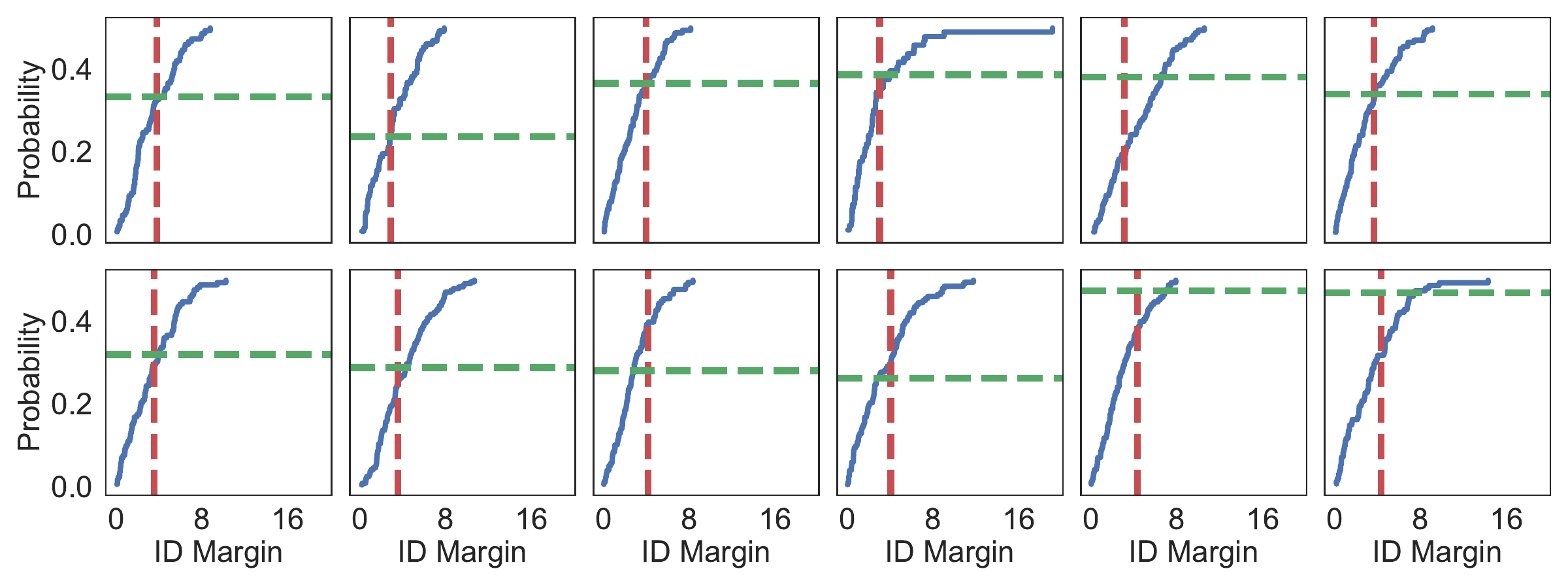}
    \caption{\small Blue line is the CDF of the \Gls{id} margin \ie \(\ip{\what}{x}\) for points \(x\) classified positively by \(\what\). The red vertical line indicates the theoretical quantity in the RHS of Condition~\eqref{ass:4-partalign}~\ie proportional to \(\tau\gamma\abs{S_k\cap S}\). The green line shows the \Gls{ood} error. Matching the result of~\Cref{corr:imbalanced}, this simulation shows that the fraction of positively classified points whose margin is less than \(\tau\gamma\abs{S_k\cap S}\) is very close to the true~\Gls{ood} error of that model. Each plot here represents a different run on a different-sized dataset.}
    \label{fig:margin-dist}\vspace{-10pt}
\end{figure}

\section{Implications and conclusion}
\label{sec:concl}
To summarise, our work argues that interpolation of label noise and presence of nuisance features break the otherwise positively correlated relationship between~\Gls{id} and~\Gls{ood} accuracy. In support of this argument, we provide experimental evidence with realistic datasets, theoretical results to isolate the sufficient conditions, and synthetic simulation to corroborate the theoretical assumptions. 

\noindent\textbf{Future Work} This raises several questions about the widely prevalent practice of preferring large but noisy datasets over smaller cleaner datasets. We hope future work will work towards striking the right balance between size and quality of datasets, keeping in mind their impact on trustworthiness metrics like~\Gls{ood} accuracy.~\Cref{thm:A1-informal} shows a sufficient condition for when label noise can increase the nuisance sensitivity; however this result is restricted to \(\min~\ell_2\) interpolators. It is an interesting question to consider what label noise models~(\eg uniform label flip) and inductive biases of the learning algorithm~(\eg~\(\min~\ell_p\))~\citep{aerni2023strong,ben2024role} can aggravate or mitigate this phenomenon. It is also interesting to investigate what other factors~(\eg spurious correlation alone~\citep{teney2023}) can lead to similar behaviours. Finally, other works have proposed approaches to mitigate memorisation of label noise~\citep{zhang2018mixup,Sanyal2020Stable}, selectively learning signal features~\citep{arjovsky2019invariant,parascandolo2021learning}, unlearning the effect of manipulated data~\citep{goel2024corrective,tanno2022repairing}, and improving robustness towards adversarial corruptions~\citep{madry2018towards,sinha2018certifiable}. The possible impact of these techniques on preventing~\awline is an interesting line of future work.

\section*{Acknowledgements}

YW is supported in part by the Office of Naval Research under grant number N00014-23-1-2590 and the National Science Foundation under grant number 2231174 and 2310831.
YY is supported by the joint Simons Foundation-NSF DMS grant \#2031899. 
YM is supported by the NSF awards: SCALE MoDL-2134209, CCF-2112665 (TILOS), as well as the DARPA AIE program, the U.S. Department of Energy, Office of Science, and CDC-RFA-FT-23-0069 from the CDC’s Center for Forecasting and Outbreak Analytics.

\bibliographystyle{plainnat}
\bibliography{acc_line}

\begin{thebibliography}{47}
\providecommand{\natexlab}[1]{#1}
\providecommand{\url}[1]{\texttt{#1}}
\expandafter\ifx\csname urlstyle\endcsname\relax
  \providecommand{\doi}[1]{doi: #1}\else
  \providecommand{\doi}{doi: \begingroup \urlstyle{rm}\Url}\fi

\bibitem[Abe et~al.(2022)Abe, Buchanan, Pleiss, Zemel, and Cunningham]{abe2022}
Taiga Abe, Estefany~Kelly Buchanan, Geoff Pleiss, Richard Zemel, and John~P Cunningham.
\newblock Deep ensembles work, but are they necessary?
\newblock In \emph{{Conference on Neural Information Processing Systems~(NeurIPS)}}, 2022.

\bibitem[Aerni et~al.(2023)Aerni, Milanta, Donhauser, and Yang]{aerni2023strong}
Michael Aerni, Marco Milanta, Konstantin Donhauser, and Fanny Yang.
\newblock Strong inductive biases provably prevent harmless interpolation.
\newblock In \emph{{International Conference on Learning Representations~(ICLR)}}, 2023.

\bibitem[Arjovsky et~al.(2019)Arjovsky, Bottou, Gulrajani, and Lopez-Paz]{arjovsky2019invariant}
Martin Arjovsky, L{\'e}on Bottou, Ishaan Gulrajani, and David Lopez-Paz.
\newblock Invariant risk minimization.
\newblock \emph{arXiv:1907.02893}, 2019.

\bibitem[Baek et~al.(2022)Baek, Jiang, Raghunathan, and Kolter]{baek2022}
Christina Baek, Yiding Jiang, Aditi Raghunathan, and J.~Zico Kolter.
\newblock Agreement-on-the-line: Predicting the performance of neural networks under distribution shift.
\newblock In \emph{{Conference on Neural Information Processing Systems~(NeurIPS)}}, 2022.

\bibitem[Bartlett et~al.(2020)Bartlett, Long, Lugosi, and Tsigler]{bartlett2020benign}
Peter~L Bartlett, Philip~M Long, G{\'a}bor Lugosi, and Alexander Tsigler.
\newblock Benign overfitting in linear regression.
\newblock \emph{Proceedings of the National Academy of Sciences}, 2020.

\bibitem[Belkin et~al.(2019)Belkin, Hsu, Ma, and Mandal]{belkin2019reconciling}
Mikhail Belkin, Daniel Hsu, Siyuan Ma, and Soumik Mandal.
\newblock Reconciling modern machine-learning practice and the classical bias--variance trade-off.
\newblock \emph{Proceedings of the National Academy of Sciences}, 2019.

\bibitem[Ben-Dov et~al.(2024)Ben-Dov, Fawkes, Samadi, and Sanyal]{ben2024role}
Omri Ben-Dov, Jake Fawkes, Samira Samadi, and Amartya Sanyal.
\newblock The role of learning algorithms in collective action.
\newblock 2024.

\bibitem[Brown et~al.(2023)Brown, Caterini, Ross, Cresswell, and Loaiza-Ganem]{brown2022verifying}
Bradley~CA Brown, Anthony~L Caterini, Brendan~Leigh Ross, Jesse~C Cresswell, and Gabriel Loaiza-Ganem.
\newblock Verifying the union of manifolds hypothesis for image data.
\newblock \emph{{International Conference on Learning Representations~(ICLR)}}, 2023.

\bibitem[Christie et~al.(2018)Christie, Fendley, Wilson, and Mukherjee]{christie2018functional}
Gordon Christie, Neil Fendley, James Wilson, and Ryan Mukherjee.
\newblock Functional map of the world.
\newblock In \emph{{IEEE Conference on Computer Vision and Pattern Recognition~(CVPR)}}, 2018.

\bibitem[Deng(2012)]{mnist}
Li~Deng.
\newblock The mnist database of handwritten digit images for machine learning research.
\newblock \emph{IEEE Signal Processing Magazine}, 2012.

\bibitem[Eastwood et~al.(2024)Eastwood, Singh, Nicolicioiu, Vlastelica~Pogan{\v{c}}i{\'c}, von K{\"u}gelgen, and Sch{\"o}lkopf]{eastwood2024spuriosity}
Cian Eastwood, Shashank Singh, Andrei~L Nicolicioiu, Marin Vlastelica~Pogan{\v{c}}i{\'c}, Julius von K{\"u}gelgen, and Bernhard Sch{\"o}lkopf.
\newblock Spuriosity didn’t kill the classifier: Using invariant predictions to harness spurious features.
\newblock \emph{{Conference on Neural Information Processing Systems~(NeurIPS)}}, 2024.

\bibitem[Frenay and Verleysen(2014)]{frenay2014class}
Benoit Frenay and Michel Verleysen.
\newblock Classification in the presence of label noise: A survey.
\newblock \emph{IEEE Transactions on Neural Networks and Learning Systems}, 2014.

\bibitem[Goel et~al.(2024)Goel, Prabhu, Torr, Kumaraguru, and Sanyal]{goel2024corrective}
Shashwat Goel, Ameya Prabhu, Philip Torr, Ponnurangam Kumaraguru, and Amartya Sanyal.
\newblock Corrective machine unlearning.
\newblock \emph{arXiv:2402.14015}, 2024.

\bibitem[Hestness et~al.(2017)Hestness, Narang, Ardalani, Diamos, Jun, Kianinejad, Patwary, Yang, and Zhou]{hestness2017deep}
Joel Hestness, Sharan Narang, Newsha Ardalani, Gregory Diamos, Heewoo Jun, Hassan Kianinejad, Md~Mostofa~Ali Patwary, Yang Yang, and Yanqi Zhou.
\newblock Deep learning scaling is predictable, empirically.
\newblock \emph{arXiv:1712.00409}, 2017.

\bibitem[Kaplan et~al.(2020)Kaplan, McCandlish, Henighan, Brown, Chess, Child, Gray, Radford, Wu, and Amodei]{kaplan2020scaling}
Jared Kaplan, Sam McCandlish, Tom Henighan, Tom~B Brown, Benjamin Chess, Rewon Child, Scott Gray, Alec Radford, Jeffrey Wu, and Dario Amodei.
\newblock Scaling laws for neural language models.
\newblock \emph{arXiv:2001.08361}, 2020.

\bibitem[Kim et~al.(2023)Kim, Sun, Raghunathan, and Kolter]{kim2023reliable}
Eungyeup Kim, Mingjie Sun, Aditi Raghunathan, and Zico Kolter.
\newblock Reliable test-time adaptation via agreement-on-the-line.
\newblock \emph{arXiv:2310.04941}, 2023.

\bibitem[Koh et~al.(2021)Koh, Sagawa, Marklund, Xie, Zhang, Balsubramani, Hu, Yasunaga, Phillips, Gao, Lee, David, Stavness, Guo, Earnshaw, Haque, Beery, Leskovec, Kundaje, Pierson, Levine, Finn, and Liang]{wilds2021}
Pang~Wei Koh, Shiori Sagawa, Henrik Marklund, Sang~Michael Xie, Marvin Zhang, Akshay Balsubramani, Weihua Hu, Michihiro Yasunaga, Richard~Lanas Phillips, Irena Gao, Tony Lee, Etienne David, Ian Stavness, Wei Guo, Berton~A. Earnshaw, Imran~S. Haque, Sara Beery, Jure Leskovec, Anshul Kundaje, Emma Pierson, Sergey Levine, Chelsea Finn, and Percy Liang.
\newblock {WILDS}: A benchmark of in-the-wild distribution shifts.
\newblock In \emph{{International Conference on Machine Learning~(ICML)}}, 2021.

\bibitem[Kumar et~al.(2022)Kumar, Raghunathan, Jones, Ma, and Liang]{kumar2022finetuning}
Ananya Kumar, Aditi Raghunathan, Robbie Jones, Tengyu Ma, and Percy Liang.
\newblock Fine-tuning can distort pretrained features and underperform out-of-distribution.
\newblock \emph{arXiv:2202.10054}, 2022.

\bibitem[Lee et~al.(2023)Lee, Moniri, Huang, Dobriban, and Hassani]{pmlr-v202-lee23o}
Donghwan Lee, Behrad Moniri, Xinmeng Huang, Edgar Dobriban, and Hamed Hassani.
\newblock Demystifying disagreement-on-the-line in high dimensions.
\newblock In \emph{{International Conference on Machine Learning~(ICML)}}, 2023.

\bibitem[Liang et~al.(2023)Liang, Mao, Kwon, Yang, and Zou]{pmlr-v202-liang23d}
Weixin Liang, Yining Mao, Yongchan Kwon, Xinyu Yang, and James Zou.
\newblock Accuracy on the curve: On the nonlinear correlation of {ML} performance between data subpopulations.
\newblock In \emph{{International Conference on Machine Learning~(ICML)}}, 2023.

\bibitem[Madry et~al.(2018)Madry, Makelov, Schmidt, Tsipras, and Vladu]{madry2018towards}
Aleksander Madry, Aleksandar Makelov, Ludwig Schmidt, Dimitris Tsipras, and Adrian Vladu.
\newblock Towards deep learning models resistant to adversarial attacks.
\newblock In \emph{{International Conference on Learning Representations~(ICLR)}}, 2018.

\bibitem[Miller et~al.(2021)Miller, Taori, Raghunathan, Sagawa, Koh, Shankar, Liang, Carmon, and Schmidt]{miller2021accuracy}
John~P Miller, Rohan Taori, Aditi Raghunathan, Shiori Sagawa, Pang~Wei Koh, Vaishaal Shankar, Percy Liang, Yair Carmon, and Ludwig Schmidt.
\newblock Accuracy on the line: on the strong correlation between out-of-distribution and in-distribution generalization.
\newblock In \emph{{International Conference on Machine Learning~(ICML)}}, 2021.

\bibitem[Miller(1981)]{kenneth1981inverse}
Kenneth~S. Miller.
\newblock On the inverse of the sum of matrices.
\newblock \emph{Mathematics Magazine}, 1981.

\bibitem[Nakkiran et~al.(2021)Nakkiran, Kaplun, Bansal, Yang, Barak, and Sutskever]{nakkiran2021deep}
Preetum Nakkiran, Gal Kaplun, Yamini Bansal, Tristan Yang, Boaz Barak, and Ilya Sutskever.
\newblock Deep double descent: Where bigger models and more data hurt.
\newblock \emph{Journal of Statistical Mechanics: Theory and Experiment}, 2021.

\bibitem[Northcutt et~al.(2021)Northcutt, Athalye, and Mueller]{northcutt2021confident}
Curtis Northcutt, Anish Athalye, and Jonas Mueller.
\newblock Confident learning: Estimating uncertainty in dataset labels.
\newblock In \emph{Journal of Artifical Intelligence Research~(JAIR)}, 2021.

\bibitem[Paleka and Sanyal(2023)]{paleka2023a}
Daniel Paleka and Amartya Sanyal.
\newblock A law of adversarial risk, interpolation, and label noise.
\newblock In \emph{{International Conference on Learning Representations~(ICLR)}}, 2023.

\bibitem[Parascandolo et~al.(2021)Parascandolo, Neitz, ORVIETO, Gresele, and Sch{\"o}lkopf]{parascandolo2021learning}
Giambattista Parascandolo, Alexander Neitz, ANTONIO ORVIETO, Luigi Gresele, and Bernhard Sch{\"o}lkopf.
\newblock Learning explanations that are hard to vary.
\newblock In \emph{{International Conference on Learning Representations~(ICLR)}}, 2021.

\bibitem[Pope et~al.(2021)Pope, Zhu, Abdelkader, Goldblum, and Goldstein]{pope2021intrinsic}
Phillip Pope, Chen Zhu, Ahmed Abdelkader, Micah Goldblum, and Tom Goldstein.
\newblock The intrinsic dimension of images and its impact on learning.
\newblock \emph{{International Conference on Learning Representations~(ICLR)}}, 2021.

\bibitem[Qiu et~al.(2023)Qiu, Kuang, and Goel]{qiu2023complexity}
GuanWen Qiu, Da~Kuang, and Surbhi Goel.
\newblock Complexity matters: Dynamics of feature learning in the presence of spurious correlations.
\newblock In \emph{NeurIPS Workshop on Mathematics of Modern Machine Learning}, 2023.

\bibitem[Sagawa et~al.(2022)Sagawa, Koh, Lee, Gao, Xie, Shen, Kumar, Hu, Yasunaga, Marklund, Beery, David, Stavness, Guo, Leskovec, Saenko, Hashimoto, Levine, Finn, and Liang]{sagawa2022extending}
Shiori Sagawa, Pang~Wei Koh, Tony Lee, Irena Gao, Sang~Michael Xie, Kendrick Shen, Ananya Kumar, Weihua Hu, Michihiro Yasunaga, Henrik Marklund, Sara Beery, Etienne David, Ian Stavness, Wei Guo, Jure Leskovec, Kate Saenko, Tatsunori Hashimoto, Sergey Levine, Chelsea Finn, and Percy Liang.
\newblock Extending the wilds benchmark for unsupervised adaptation.
\newblock In \emph{{International Conference on Learning Representations~(ICLR)}}, 2022.

\bibitem[Sanyal et~al.(2020)Sanyal, Torr, and Dokania]{Sanyal2020Stable}
Amartya Sanyal, Philip~H. Torr, and Puneet~K. Dokania.
\newblock Stable rank normalization for improved generalization in neural networks and gans.
\newblock In \emph{{International Conference on Learning Representations~(ICLR)}}, 2020.

\bibitem[Sanyal et~al.(2021)Sanyal, Dokania, Kanade, and Torr]{sanyal2021how}
Amartya Sanyal, Puneet~K. Dokania, Varun Kanade, and Philip Torr.
\newblock How benign is benign overfitting ?
\newblock In \emph{{International Conference on Learning Representations~(ICLR)}}, 2021.

\bibitem[Sch{\"o}lkopf(2019)]{scholkopf2022causality}
Bernhard Sch{\"o}lkopf.
\newblock Causality for machine learning.
\newblock 2019.
\newblock \doi{10.1145/3501714.3501755}.

\bibitem[Shah et~al.(2020)Shah, Tamuly, Raghunathan, Jain, and Netrapalli]{shah2020pitfalls}
Harshay Shah, Kaustav Tamuly, Aditi Raghunathan, Prateek Jain, and Praneeth Netrapalli.
\newblock The pitfalls of simplicity bias in neural networks.
\newblock \emph{{Conference on Neural Information Processing Systems~(NeurIPS)}}, 2020.

\bibitem[Shi et~al.(2023)Shi, Daunhawer, Vogt, Torr, and Sanyal]{shi2023how}
Yuge Shi, Imant Daunhawer, Julia~E Vogt, Philip Torr, and Amartya Sanyal.
\newblock How robust is unsupervised representation learning to distribution shift?
\newblock In \emph{The Eleventh International Conference on Learning Representations}, 2023.

\bibitem[Singla and Feizi(2021)]{singla2021salient}
Sahil Singla and Soheil Feizi.
\newblock Salient imagenet: How to discover spurious features in deep learning?
\newblock In \emph{{International Conference on Learning Representations~(ICLR)}}, 2021.

\bibitem[Sinha et~al.(2018)Sinha, Namkoong, and Duchi]{sinha2018certifiable}
Aman Sinha, Hongseok Namkoong, and John Duchi.
\newblock Certifiable distributional robustness with principled adversarial training.
\newblock In \emph{{International Conference on Learning Representations~(ICLR)}}, 2018.

\bibitem[Tanno et~al.(2022)Tanno, Pradier, Nori, and Li]{tanno2022repairing}
Ryutaro Tanno, Melanie~F. Pradier, Aditya Nori, and Yingzhen Li.
\newblock Repairing neural networks by leaving the right past behind.
\newblock In \emph{{Conference on Neural Information Processing Systems~(NeurIPS)}}, 2022.

\bibitem[Teney et~al.(2023)Teney, Lin, Oh, and Abbasnejad]{teney2023}
Damien Teney, Yong Lin, Seong~Joon Oh, and Ehsan Abbasnejad.
\newblock Id and ood performance are sometimes inversely correlated on real-world datasets.
\newblock In \emph{{Conference on Neural Information Processing Systems~(NeurIPS)}}, 2023.

\bibitem[Tripuraneni et~al.(2021{\natexlab{a}})Tripuraneni, Adlam, and Pennington]{tripuraneni2021}
Nilesh Tripuraneni, Ben Adlam, and Jeffrey Pennington.
\newblock Overparameterization improves robustness to covariate shift in high dimensions.
\newblock In \emph{{Conference on Neural Information Processing Systems~(NeurIPS)}}, 2021{\natexlab{a}}.

\bibitem[Tripuraneni et~al.(2021{\natexlab{b}})Tripuraneni, Adlam, and Pennington]{tripuraneni2021covariate}
Nilesh Tripuraneni, Ben Adlam, and Jeffrey Pennington.
\newblock Covariate shift in high-dimensional random feature regression.
\newblock \emph{arXiv:2111.08234}, 2021{\natexlab{b}}.

\bibitem[Vershynin(2018)]{Vershynin_2018}
Roman Vershynin.
\newblock \emph{High-Dimensional Probability: An Introduction with Applications in Data Science}.
\newblock 2018.

\bibitem[Wang et~al.(2021)Wang, Liu, and Levy]{wang2021fair}
Jialu Wang, Yang Liu, and Caleb Levy.
\newblock Fair classification with group-dependent label noise.
\newblock In \emph{{ACM conference on fairness, accountability, and transparency~(FaccT)}}, 2021.

\bibitem[Wenzel et~al.(2022)Wenzel, Dittadi, Gehler, Simon-Gabriel, Horn, Zietlow, Kernert, Russell, Brox, Schiele, Sch\"{o}lkopf, and Locatello]{wenzel2022}
Florian Wenzel, Andrea Dittadi, Peter Gehler, Carl-Johann Simon-Gabriel, Max Horn, Dominik Zietlow, David Kernert, Chris Russell, Thomas Brox, Bernt Schiele, Bernhard Sch\"{o}lkopf, and Francesco Locatello.
\newblock Assaying out-of-distribution generalization in transfer learning.
\newblock In \emph{{Conference on Neural Information Processing Systems~(NeurIPS)}}, 2022.

\bibitem[Wu et~al.(2022)Wu, Gong, Han, Liu, and Liu]{wu2022fair}
Songhua Wu, Mingming Gong, Bo~Han, Yang Liu, and Tongliang Liu.
\newblock Fair classification with instance-dependent label noise.
\newblock In \emph{{Conference on Causal Learning and Reasoning}~(CLeaR)}, 2022.

\bibitem[Zhang et~al.(2017)Zhang, Bengio, Hardt, Recht, and Vinyals]{zhang2017understanding}
Chiyuan Zhang, Samy Bengio, Moritz Hardt, Benjamin Recht, and Oriol Vinyals.
\newblock Understanding deep learning requires rethinking generalization.
\newblock In \emph{{International Conference on Learning Representations~(ICLR)}}, 2017.

\bibitem[Zhang et~al.(2018)Zhang, Cisse, Dauphin, and Lopez-Paz]{zhang2018mixup}
Hongyi Zhang, Moustapha Cisse, Yann~N. Dauphin, and David Lopez-Paz.
\newblock mixup: Beyond empirical risk minimization.
\newblock In \emph{{International Conference on Learning Representations~(ICLR)}}, 2018.

\end{thebibliography}

\clearpage
\appendix
\section{Proofs}
\label{app:proof}
In this section, we provide the proofs of~\Cref{thm:main-thm} and~\Cref{corr:imbalanced}. Then, we state and prove~\Cref{thm:A1}, which is the full version of~\Cref{thm:A1-informal}.

\mainthm*
\begin{proof}
WLOG consider \(x\in\mathrm{dom}\br{\dist}\) such that
\(\ip{\what}{x}\geq 0\). Then, define the event \( E := \bc{\delta \in
\mathbb{R}^{d+k} : \ip{\what}{x + \delta} \leq 0} \) for which we need
to bound \(\prob\bs{E}\). Decomposing the dot product affected by the
shift, \(\ip{\what}{x + \delta} = \ip{\what}{x} +
\ip{\what}{\delta}\).  Given \(\shift\) is \(S_d\)-oblivious,
\(\delta\) contributes only from the coordinates in \(S_k \cap S\).
Then, we can simplify the probability as

\begin{equation}\label{ineq:prob-defn-event}
\begin{aligned}
\prob\bs{E} =\prob_{\delta}\bs{\ip{\hat{w}}{x + \delta} \leq 0}&=\prob_{\delta}\bs{\sum_{i \in S_k \cap S} \hat{w}_i \delta_i \leq -\ip{\what}{x}}\\
&= 1 - \prob_{\delta}\bs{\sum_{i \in S_k \cap S} \hat{w}_i \delta_i \geq -\cC}\\
&\geq 1 - \inf_{\lambda\geq 0}e^{\lambda\cC}~\bE\bs{e^{\lambda\sum_{i \in S_k \cap S} \hat{w}_i \delta_i }}
\end{aligned}
\end{equation}

As each \(\delta_i\) is independently distributed with subgaussian parameter \(\sigma\) and mean \(\nu_i\), by the properties of subgaussian random variables, for \(\lambda \geq 0\), the moment generating function (MGF) of \(\delta_i\) yields

\[
    \bE\bs{e^{\lambda \delta_i}} \leq e^{\lambda \nu_i + \frac{\lambda^2 \sigma^2}{2}}
\]

Subsituting this into the Chernoff bound in~\Cref{ineq:prob-defn-event}, we obtain

\begin{equation}\label{subst-mgf-into-chernoff}
\begin{aligned}    
\prob\bs{E}&\geq 1 - \inf_{\lambda\geq 0}e^{\lambda\cC+\lambda\ip{\what}{\nu}+\lambda^2\sum_{i\in S_k\cap S}\nicefrac{\what_i^2 \sigma^2}{2}}\\
&\geq 1 - \inf_{\lambda\geq 0}e^{\lambda\cC+\abs{S_k\cap S}\br{-\lambda\tau\gamma+\lambda^2 M^2\sigma^2 }}
\end{aligned}
\end{equation}
where the last inequality uses Assumptions A1 and A2. Solving the above optimisation to obtain the optimal lambda yields
\[\lambda = \frac{\tau\gamma\abs{S_k\cap S} - \cC}{\sigma^2M^2}\]

Note that assumption A3 ensures that this term is positive. Substituting this back into~\Cref{subst-mgf-into-chernoff} and simplifying the resultant expression yields the following probability bound

\[
\prob\br{\ip{\hat{w}}{x + \delta} \leq 0} \geq 1 - \exp\bc{-\frac{\abs{S_k \cap S}\br{\tau\gamma - \nicefrac{\cC}{\abs{S_k\cap S}}}^2 }{2 \sigma^2 M^2}}.
\]

\end{proof}

Here, we provide a full version of~\Cref{corr:imbalanced}. Specifically,~\Cref{corr:imbalanced} is a special case of~\Cref{corr:imbalanced-full} with $c = 1/2$. 
\begin{restatable}{corollary}{CorrImbalanced}\label{corr:imbalanced-full} Define \(S_d,S_k\), and \(\dist\) as in~\Cref{thm:main-thm} and \(\shift^\prime\) as described above. Consider any \(\what \in
\mathbb{R}^{d+k}\) with support \(S\) such that \(\what\) satisfies
Conditions~\ref{ass:1-sens} and \ref{ass:2-align}. For any
\(0\leq c\leq 1\) define
\begin{equation}\label{ass:4-partalign}\rho_c=\max_{\hat{y}\in\bc{-1,1}}\prob_{x\sim\dist}\bs{\bI\bc{\hat{y}\ip{\what}{x}\geq 0}\cdot \bI\bc{\hat{y}\ip{\what}{x}\leq c \tau\gamma_{\hat{y}} \abs{S_k \cap S}}}.\tag{C4}\end{equation}
Then, we have
    \[
    \prob_{x,\delta}\br{\ip{\hat{w}}{x + \delta} \neq \ip{\what}{x}} \geq \rho_c \sum_{i\in \{-1, 1\}}c_{i}\br{1 - \exp\bc{-\frac{\abs{S_k \cap S}\tau^2\gamma_{i}^2\br{ 1- c}^2 }{2 \sigma^2 M^2}}}.
    \]
\end{restatable}

\begin{proof}
Let $y_{\max}$ denote the value of $\hat{y}$ that achieve $\rho_c$, ie.\[y_{\max} = \argmax_{\hat{y}\in\bc{-1,1}}\prob_{x\sim\dist}\bs{\bI\bc{\hat{y}\ip{\what}{x}\geq 0}\cdot \bI\bc{\hat{y}\ip{\what}{x}\leq c \tau\gamma_{\hat{y}} \abs{S_k \cap S}}}.\]

For simplicity, we also denote the event $\bI\bc{\ba{\hat{w}, x + \delta} \neq \ba{\hat{w}, x}}$ as $\cE_{\textrm{err}}$. Then, the goal of the proof is to lower bound $\prob_{x, \delta\sim \Delta'}\br{\cE_{\textrm{err}}}$. Let $\cE$ denote the event $\bI\bc{y_{\max}\ip{\what}{x}\geq 0}\cdot \bI\bc{y_{\max}\ip{\what}{x}\leq c \tau\gamma \abs{S_k \cap S}}$. We apply law of total probability over the event $\cE$ and $\cE^c$
\begin{equation}\label{eq:corollary-ineq1}
\begin{aligned}
      \prob_{x\sim \dist, \delta\sim \Delta'}\br{\cE_{\textrm{err}}} &= \prob_{x\sim \dist}\br{\cE}\prob_{x\sim\mu|\cE, \delta\sim \Delta'}\br{\cE_{\textrm{err}}} + \prob_{x\sim \dist}\br{\cE^c} \prob_{x\sim \dist|\cE^c, \delta\sim \Delta'}\br{\cE_{\textrm{err}}}\\
      &\geq \rho_c\prob_{x\sim \dist|\cE, \delta\sim \Delta'}\br{\cE_{\textrm{err}}} \\
\end{aligned}
\end{equation}
WLOG, we assume $y_{\max} = 1$. Then, we rewrite the probability $\prob_{x|\cE, \delta\sim \Delta'}\br{\ba{\hat{w}, x + \delta} \neq \ba{\hat{w}, x}}$ and invoke~\Cref{thm:main-thm} to lower bound the term. As $y_{\max} = 1$, $\ba{\hat{w}, x}\geq 0$ when event $\cE$ holds. In other words, $\ba{\hat{w}, x}\geq 0$ for all $x$ in the support of $\mu|\cE$. Hence, by law of total probability over the the mixture distributions $\Delta_1$ and $\Delta_{-1}$ of the shift random variable $\delta$, 
\begin{equation}\label{eq:corollary-ineq2}
    \begin{aligned}
    \prob_{x\sim \dist|\cE, \delta\sim \Delta'}\br{\cE_{\textrm{err}}} &= \prob_{x\sim\dist|\cE, \delta\sim \Delta'}\br{\ba{\hat{w}, x + \delta} \leq 0}\\
    &= c_1\prob_{x\sim\dist|\cE, \delta\sim \Delta_1}\br{\ba{\hat{w}, x + \delta} \leq 0} + c_{-1}\prob_{x\sim\dist|\cE, \delta\sim \Delta_{-1}}\br{\ba{\hat{w}, x + \delta} \leq 0}\\
    \end{aligned}
\end{equation}

Then, we derive lower bounds on $\prob_{x|\cE, \delta\sim \Delta_1}\br{\ba{\hat{w}, x + \delta} \leq 0}$ and $\prob_{x|\cE, \delta\sim \Delta_{-1}}\br{\ba{\hat{w}, x + \delta} \leq 0}$ respectively using~\Cref{thm:main-thm}. By the definition of event $\cE$ and the fact that $y_{\max} = 1$, any $x$ in the support of $\mu|\cE$ satisfies $\ba{\hat{w}, x}\leq c\tau\gamma\abs{S_k\cap S}$. We then employ~\Cref{thm:main-thm} with $\cC = c\tau\gamma\abs{S_k\cap S}$ and $\gamma = \gamma_{+1}$ to lower bound $\prob_{x|\cE, \delta\sim \Delta_1}\br{\ba{\hat{w}, x + \delta} \leq 0} $ when $\ba{\hat{w}, x}\geq 0$, 
\begin{equation}\label{eq:corollary-ineq3}
    \prob_{x|\cE, \delta\sim \Delta_1}\br{\ba{\hat{w}, x + \delta} \leq 0} \geq 1-\exp\bc{-\frac{|S_k\cap S|\tau^2\gamma_1^2 (1-c)^2}{2\sigma^2M^2}}
\end{equation}
Similarly, setting $\cC = c\tau\gamma\abs{S_k\cap S}$ and $\gamma = \gamma_{-1}$ we can lower bound $\prob_{x|\cE, \delta\sim \Delta_{-1}}\br{\ba{\hat{w}, x + \delta} \leq 0} $ when $\ba{\hat{w}, x}\geq 0$ by
\begin{equation}\label{eq:corollary-ineq3}
    \prob_{x|\cE, \delta\sim \Delta_{-1}}\br{\ba{\hat{w}, x + \delta} \leq 0} \geq 1-\exp\bc{-\frac{|S_k\cap S|\tau^2\gamma_{-1}^2 (1-c)^2}{2\sigma^2M^2}}
\end{equation}
Substituting~\Cref{eq:corollary-ineq3} into~\Cref{eq:corollary-ineq2} and then substituting~\Cref{eq:corollary-ineq2} into~\Cref{eq:corollary-ineq1}, we obtain     
\[
    \prob_{x,\delta}\br{\ip{\hat{w}}{x + \delta} \neq \ip{\what}{x}} \geq \rho_c \sum_{i\in \{-1, 1\}}c_{i}\br{1 - \exp\bc{-\frac{\abs{S_k \cap S}\tau^2\gamma_{i}^2\br{ 1- c}^2 }{2 \sigma^2 M^2}}}.
    \]
for $y_{\max} = 1$. The case with $y_{\max} = -1$ follows the same analysis.
\end{proof}

Now we state the full version of~\Cref{thm:A1-informal} as~\Cref{thm:A1} and provide its proof. 

\begin{thm}\label{thm:A1}
    If $d = \Omega(\log n)$ and the noise distribution $\pi$ satisfies $\Pr_{\xi\sim \pi^n, X_{2:d}}[X_{2:d}^+\xi\geq C]\geq 1-\beta$ for some constant $C$ and $\beta \in (0, 1)$. Then, for $\beta_1, \beta_2\in (0, 1)$ and $\beta_1 = O(1/d)$, with probability at least $0.92 - \beta_2 - d\beta_1 - \beta$, the min-$\ell_2$-interpolator $\hat{w}$ on the noisy dataset $(X, Y)$ satisfies 
     \[\norm{\hat{w}_{2:d}}_\infty\geq 0.1\br{1-\frac{2\log\frac{n}{\beta_1}}{(d-1)\br{1-\sqrt{6\beta_2\log\frac{\beta_1}{d-1}}}}}C.\]
\end{thm}

\begin{proof}
         
    For $i\in \{1, ..., d\}$, let $X_i\in \reals^{n}$ denote the $i^{\text{th}}$ feature of the data matrix $X$. For simplicity, let $\tilde{X} = X_{2:d}\in \reals^{n\times (d-1)}$ denote the nuisance covariates, $\Sigma_1 = X_1X_1^\top$ and $\tilde{\Sigma} = \tilde{X}\tilde{X}^\top$. When $d > n$, min-$\ell_2$-interpolator on the noisy dataset $(X, Y)$ has a closed-form solution $
        \hat{w} = X^\top \br{XX^\top}^{-1}Y$. We will show that the estimated parameters for nuisance features $\hat{w}_{2:d}$ is lower bounded with high probability. 
    \begin{equation}
        \label{eq:min-l2-interpolator-closed-form-solution}
        \begin{aligned}
            \hat{w}_{2:d} &= \bs{X^\top \br{XX^\top}^{-1}Y }_{2:d} \\
            &\overset{(a)}{=} \tilde{X}^\top\br{\Sigma_1 + \tilde{\Sigma}}^{-1}\xi\odot X_1 \\
           &\overset{(b)}{=} \tilde{X}^T\br{\tilde{\Sigma}^{-1} -\frac{\tilde{\Sigma}^{-1}\Sigma_1\tilde{\Sigma}^{-1}}{1 + Tr(\Sigma_1\tilde{\Sigma}^{-1})} }^{-1}(\xi\odot X_1) \\
           &= \underbrace{\tilde{X}^\top \tilde{\Sigma}^{-1}\xi\odot X_1}_{\text{Part I}} -\underbrace{ \tilde{X}^T\frac{\tilde{\Sigma}^{-1}\Sigma_1\tilde{\Sigma}^{-1}}{1 + Tr(\Sigma_1\tilde{\Sigma}^{-1})} \xi \odot X_1}_{\text{Part II}}\\
        \end{aligned}
    \end{equation}  
    where step (a) follows from the definition of $\Sigma_1, \tilde{\Sigma}$ and $Y$, and step (b) follows from~\Cref{lem:inverse-of-sum-of-matrices}~\citep{kenneth1981inverse}.
    \begin{lem}[Inverse of sum of matrices~\citep{kenneth1981inverse}]\label{lem:inverse-of-sum-of-matrices}
        For two matrices $A$ and $B$, let $g = Tr(BA^{-1})$. If $A$ and $A + B$ are invertible and $B$ has rank $1$, then $g\neq -1$ and \[(A + B)^{-1} = A^{-1} - \frac{1}{g + 1}A^{-1}B A^{-1}. \]
    \end{lem}
     We will lower bound Part I and Part II separately using the assumption on the noise distribution $\mu$ and the concentration bound on Gaussian random matrix. 

     Applying the assumption on the noise distribution and the fact that $\tilde{X}^\top \tilde{\Sigma}^{-1} = X_{2:d}^+$, we can lower bound Part I with probability at least $1-\beta$, 
     \begin{equation}
         \label{eq:lower-bound-part-i}
         \tilde{X}^\top \tilde{\Sigma}^{-1}\xi\odot X_1 = X_{2:d}^+ \xi \odot X_1 \geq C \odot X_1 \geq C X_1. 
     \end{equation}
    It remains to upper bound Part II. 
    \begin{equation}\label{eq:lower-bound-part-ii}
        \begin{aligned}
            \tilde{X}^T\frac{\tilde{\Sigma}^{-1}\Sigma_1\tilde{\Sigma}^{-1}}{1 + Tr(\Sigma_1\tilde{\Sigma}^{-1})} \xi \odot X_1 &\leq \tilde{X}^\top \tilde{\Sigma}^{-1}\xi\odot X_1 \norm{\frac{\Sigma_1\tilde{\Sigma}^{-1}}{1 + Tr(\Sigma_1\tilde{\Sigma}^{-1})}}_{op}\\
            &\leq C \norm{\frac{\Sigma_1\tilde{\Sigma}^{-1}}{1 + Tr(\Sigma_1\tilde{\Sigma}^{-1})}}_{op} \leq C\norm{\Sigma_1\tilde{\Sigma}^{-1}}_{op}
        \end{aligned}
    \end{equation}
    where the last inequality follows from $Tr(\Sigma_1\tilde{\Sigma}^{-1}) \geq 0$ for positive semi-definite matrices $\Sigma_1$ and $\tilde{\Sigma}$. 
    
     Then, we derive lower bound and upper bound on the term $\norm{\frac{1}{n}\Sigma_1}_{op}$ and $\norm{\br{\frac{1}{n}\tilde{\Sigma}}^{-1}}_{op}$, 
    \begin{equation}
       \norm{\frac{1}{n}\Sigma_1}_{op} = \frac{1}{n}X_1^\top X_1 = \frac{1}{n}\sum_{i = 1}^n X_{1i}^2
    \end{equation}
    where each $X_{1i}\sim \cN(0, 1)$.
    \begin{lem}[Tail bound of norm of Gaussian random vector]\label{lem:gaussian-norm-tail-bound}
    Let $X = (X_1, ..., X_n)\in \reals^n$ be a vector where each $X_i$ is an independent standard Gaussian random variable, then for some constant $C\geq 0$, \[\prob\bs{\frac{1}{n}\sum_{i = 1}^n X_i^2\leq C}\geq 1-ne^{-\nicefrac{nC}{2}}\] 
    \end{lem}
    
    Thus, with probability $1-\beta_1$ for $\beta_1\in (0, 1)$, 
    \begin{equation}\label{eq:operation-norm-bounds1}
        \prob\bs{\norm{\Sigma_1}_{op}\leq 2\log\frac{n}{\beta_1}}=\prob\bs{\frac{1}{n}\norm{\Sigma_1}_{op}\leq \frac{2}{n}\log\frac{n}{\beta_1}}\geq1-\beta_1.
    \end{equation}
        Let $\lambda_i(\Sigma)$ denote the $i^{\text{th}}$ eigenvalue of a matrix $\Sigma$. Then, $\norm{\br{\frac{1}{n}\tilde{\Sigma}}^{-1}}_{op} = \br{\min_{i}\lambda_i\br{\frac{1}{n}\tilde{\Sigma}}}^{-1}$, we only need to lower bound $\min_{i}\lambda_i\br{\frac{1}{n}\tilde{\Sigma}}$,   
    \begin{equation}
        \label{eq:operation-norm-prep}
        \begin{aligned}
                \prob\bs{\abs{\min_{i}\lambda_i\br{\frac{1}{d-1}\tilde{\Sigma}} - 1}\leq t} &\geq \prob \bs{\forall i, \abs{\lambda_i\br{\frac{1}{d-1}\tilde{\Sigma}} - 1}\leq t}\\
                &= \prob \bs{\max_i\abs{\lambda_i\br{\frac{1}{d-1}\tilde{\Sigma}} - 1}\leq t} \\
                &\overset{(a)}{\geq}\prob\bs{\norm{\frac{1}{d-1}\tilde{\Sigma} - I }_{op}\leq \sqrt{6\beta_2\log \frac{\beta_1}{(d-1)}}}\\
                &\overset{(b)}{\geq}1 - \beta_2 - (d-1)\beta_1
        \end{aligned}
    \end{equation}
       where the first inequality follows from Weyl's inequality (\Cref{lem:weyl-ineq}), and step (b) follows a corollary of matrix Bernstein inequality (\Cref{cor:modified-matrix-bernstein}) by setting $t = \sqrt{6\beta_2\log \frac{\beta_1}{(d-1)}}$. 
    \begin{lem}[Weyl's inequality~\citep{Vershynin_2018}]\label{lem:weyl-ineq}
        For any two symmetric matrices $A, B$ with the same dimension, \[\max_{i}\abs{\lambda_i(A)-\lambda_i(B)}\leq \norm{A - B}_{op}\]
    \end{lem}
        \begin{corollary}\label{cor:modified-matrix-bernstein}
        Let $X_1, ..., X_d$ be independent Gaussian random vectors in $\bR^n$ with mean $0$ and covariance matrix $I_n$. Then for all $t\geq 0$ and $\beta\in (0, 1)$, \[\prob\bs{\norm{\frac{1}{d}\sum_{i = 1}^d X_iX_i^T - I_n}_{op}\leq t}\geq 1-2n\exp\br{\frac{-dt^2}{4\log\frac{n}{\beta}\br{1 + \nicefrac{2t}{3}}}}-d\beta\]
    \end{corollary}

    Therefore, with probability $1-\beta_2 - d\beta_1$, 
    \begin{equation}\label{eq:upperbound-norm}
        \norm{\Sigma_1\tilde{\Sigma}^{-1}}_{op}\leq \norm{\Sigma_1}_{op}\norm{\tilde{\Sigma}^{-1}}_{op}\leq \frac{2\log\frac{n}{\beta_1}}{(d-1)\br{1-\sqrt{6\beta_2\log\frac{\beta_1}{d-1}}}},
    \end{equation}
    where the last inequality is obtained by substituting the upper bound for $\norm{\Sigma_1}_{op}$ and $\norm{\tilde{\Sigma}}_{op}$ in~\Cref{eq:operation-norm-prep} and~\Cref{eq:operation-norm-bounds1} respectively.

    Substituting~\Cref{eq:upperbound-norm} into~\Cref{eq:lower-bound-part-ii}, we obtain an upper bound on Part II, 
    \begin{equation}
        \label{eq:lower-bound-part-ii-final}
        \tilde{X}^T\frac{\tilde{\Sigma}^{-1}\Sigma_1\tilde{\Sigma}^{-1}}{1 + Tr(\Sigma_1\tilde{\Sigma}^{-1})} \xi \odot X_1 \leq \frac{2CX_1\log\frac{n}{\beta_1}}{(d-1)\br{1-\sqrt{6\beta_2\log\frac{\beta_1}{d-1}}}}. 
    \end{equation}
    
    Combining the lower bound on Part I (\Cref{eq:lower-bound-part-i}) and the upper bound on Part II (\Cref{eq:lower-bound-part-ii-final}), we get the following lower bound in terms of the random vector $X$, 
    \begin{equation}
        \label{eq:lower-bound-in-terms-of-X1}
        \hat{w}_{2:d}\geq \br{1-\frac{2\log\frac{n}{\beta_1}}{(d-1)\br{1-\sqrt{6\beta_2\log\frac{\beta_1}{d-1}}}}}C X_1
    \end{equation}

    Finally, we employ anti-concentration bound of standard normal random variable to lower bound $\norm{X_1}_\infty$ to conclude the proof. 

    By calculation with the cumulative distribution function of the standard normal random variable, with probability at least 0.92, there exists $i$ such that $\abs{X_{1i}}\geq 0.1$. Therefore, with probability at least $0.92 - \beta_2 - d\beta_1 - \beta$, 
    \[\norm{\hat{w}_{2:d}}_\infty\geq 0.1\br{1-\frac{2\log\frac{n}{\beta_1}}{(d-1)\br{1-\sqrt{6\beta_2\log\frac{\beta_1}{d-1}}}}}C.\]
    
    This concludes the proof. 
    
\end{proof}

\begin{proof}[Proof of~\Cref{lem:gaussian-norm-tail-bound}]
    \begin{equation}
        \begin{aligned}
            \prob\bs{\frac{1}{n}\sum_{i = 1}^nX_i^2\leq C} &= \prob\bs{\sum_{i = 1}^nX_i^2\leq nC}\\
            &\geq \prob\bs{\forall i, X_i^2\leq nC} = \prob\bs{\forall i, X_i \geq \sqrt{nC}}\\
            &= 1-\prob\bs{\exists i, X_i\geq \sqrt{nC}}\\
            &\geq 1-n e^{-\nicefrac{nC}{2}}
        \end{aligned}
    \end{equation}
    where the last inequality follows from Union bound and the tail bound of a standard Gaussian random variable. 
\end{proof}
\begin{proof}[Proof of~\Cref{cor:modified-matrix-bernstein}]
    We first apply~\Cref{lem:gaussian-norm-tail-bound} with $C = 2\log\frac{n}{\beta}$. That is, with probability at least $1-d$, all $X_1, ..., X_d\in \reals^n$ satisfy 
    \[\prob\bs{\forall i\in [d], \norm{X_i}_2\leq 2\log\frac{n}{\beta}}\geq 1-d\beta. \]

    Applying a standard corollary of Matrix Bernstein inequality\footnote{See Theorem 13.5 in the lecture notes: \url{https://www.stat.cmu.edu/~arinaldo/Teaching/36709/S19/Scribed_Lectures/Mar5_Tim.pdf}} on $X_1, ..., X_d$ concludes the proof. 
\end{proof}

\section{Experimental details}
\label{app:exp-det}

\subsection{Data Distribution for simulation using synthetic data}\label{app:synthetic-exp-setting}
We construct a binary classification task in a \(300\)-dimensional
space. The procedure for generating the training dataset is as
follows: Each label \(y \in \{-1, 1\}\) is sampled uniformly at
random. The first component \(x_1\) is sampled from a mixture of two Gaussian distributions with a variance of
$0.15$, centered at $y$ and $1-y$ respectively, with mixing proportions of $0.9$ and $0.1$. As the training dataset size increases, the model's ability to learn this feature improves, thereby improving the test accuracy. The remaining
\(299\) dimensions (\(x_2,\ldots,x_{300}\)) are drawn from a standard
normal distribution with zero mean and a variance of \(0.1\). They
constitute the nuisance subspace, primarily used to memorise label
noise. We introduce label noise into training data by flipping 20\% of the labels.

For in-distribution (ID) accuracy evaluation, we generate a fresh set
of data points from the initial distribution, devoid of label noise. Out-of-distribution
(OOD) accuracy is evaluated by first constructing the shift distribution \(\shift\). Assuming \(\what\) represents the trained linear model, the mean \(\nu\) of the \(\shift\) distribution for each component \(i\) for \(i>2\) is set to \(-0.25 \sgn{\what_i}\) with a variance of \(10^{-3}\). We then simulate a new~\Gls{id} test instance \(z, y\) in the usual manner, sample a shift \(\delta \sim \shift\),
and add them \(z + \delta\) to generate the OOD test point.

All plots related to linear synthetic experiments can be generated in a total of less than two hours on a Macbook Pro M2.

\subsection{Additional results on synthetic linear setting}

In this section, we provide new results in~\Cref{fig:vary-reg} where we vary the regularisation strength. The results show that both~\Gls{id} and~\Gls{ood} accuracy increases with increasing regularisation coefficeint but larger datasets have a noticeably smaller~\Gls{ood} accuracy in the noisy setting uniformly across all regularisation strengths. For all other cases, including noiseless~\Gls{ood} and both noisy and noiseless~\Gls{id}, larger datasets perform better. The results also show that regularisation affects nuisance density and sensitivity as expected\ie larger regularisation leads to lower sensitivity and density. But both the sensitivity and density falls to zero faster for the noiseless setting compared to the noisy setting.

\begin{figure}[!htb]\centering
\begin{subfigure}[t]{0.15\linewidth}
    \centering
    \includegraphics[width=1.0\linewidth]{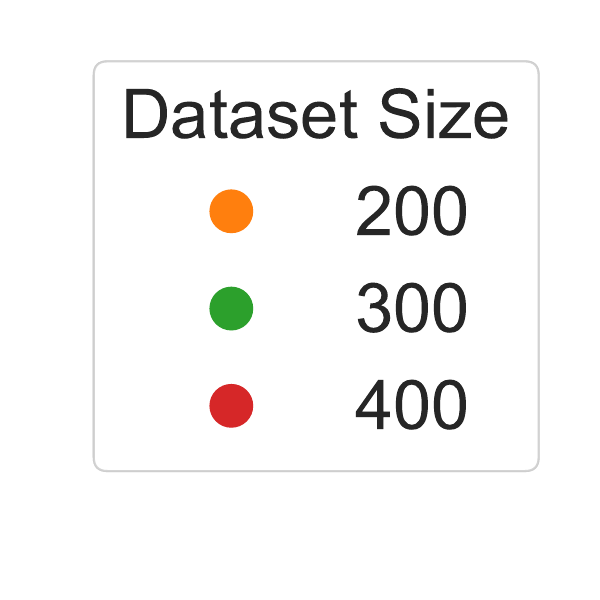}
    \includegraphics[width=1.0\linewidth]{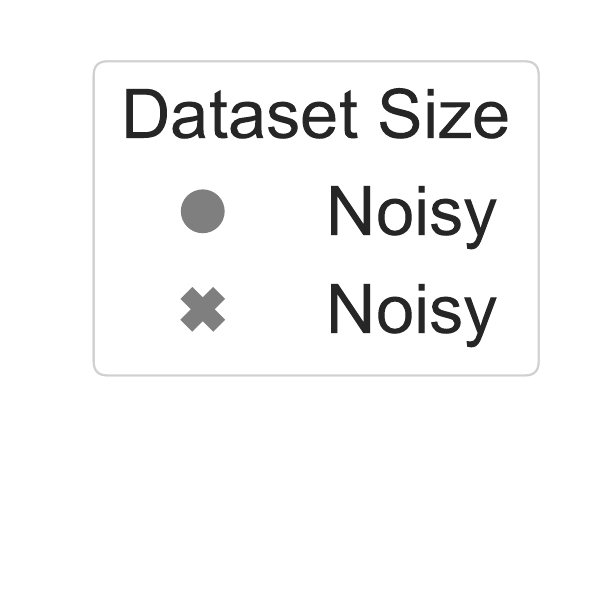}
\end{subfigure}
\begin{subfigure}[t]{0.4\linewidth}
    \centering    \includegraphics[width=0.48\linewidth]{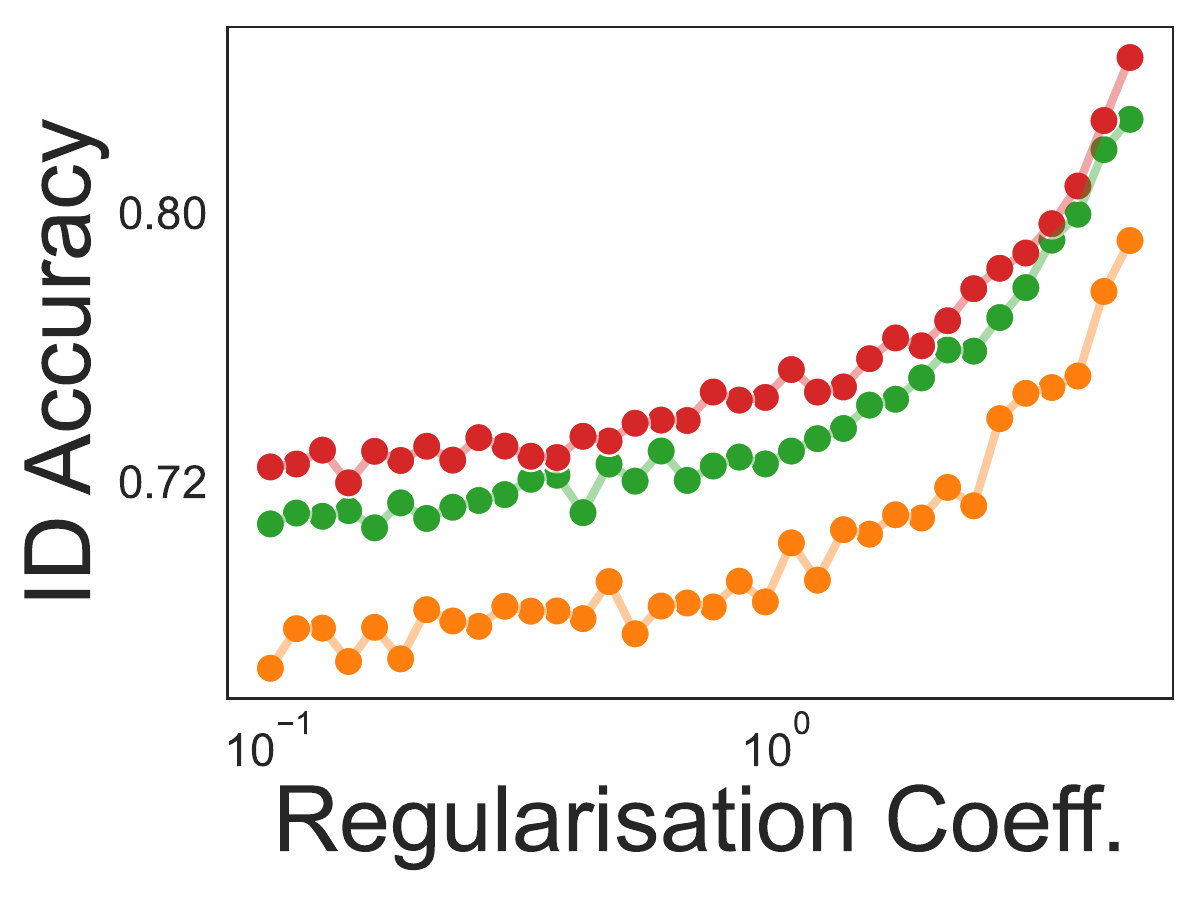}   
    \includegraphics[width=0.48\linewidth]{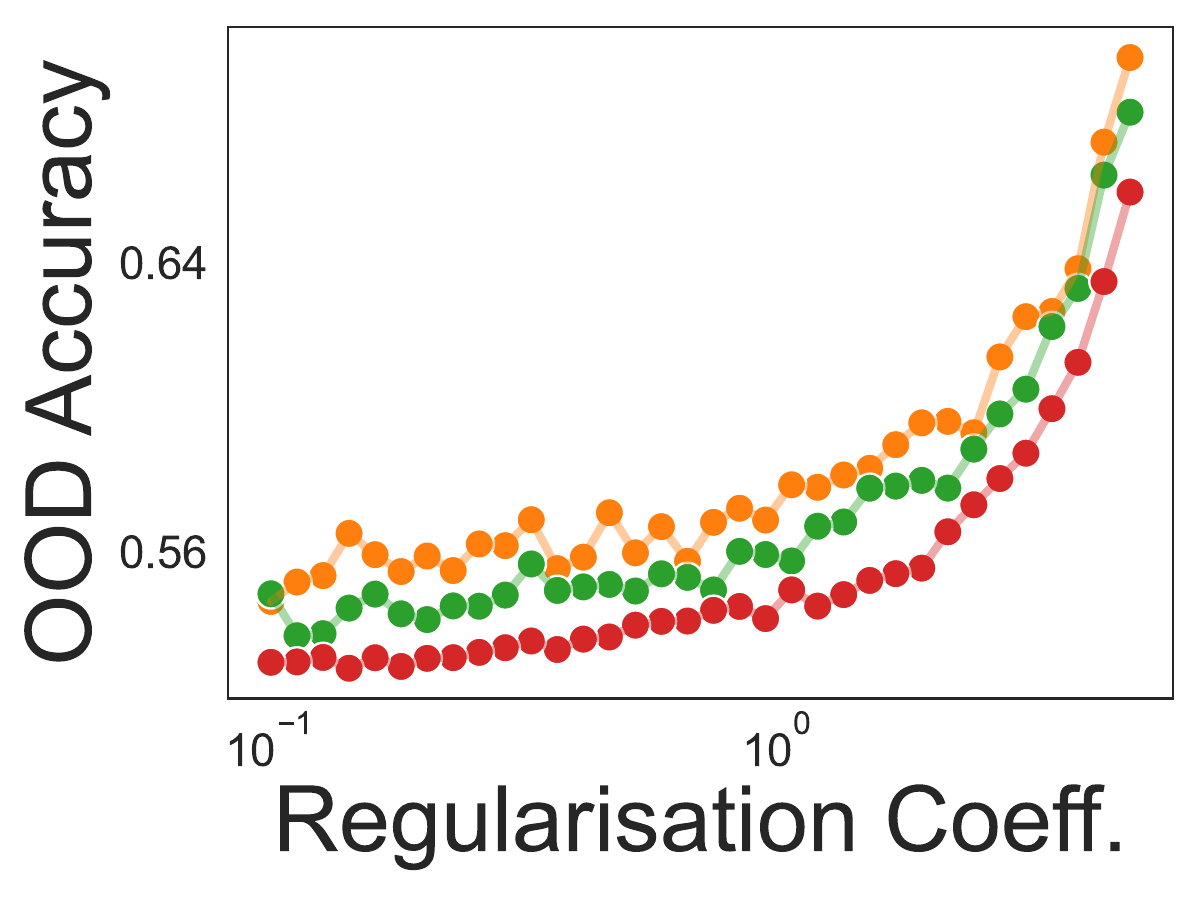}
    \includegraphics[width=0.48\linewidth]{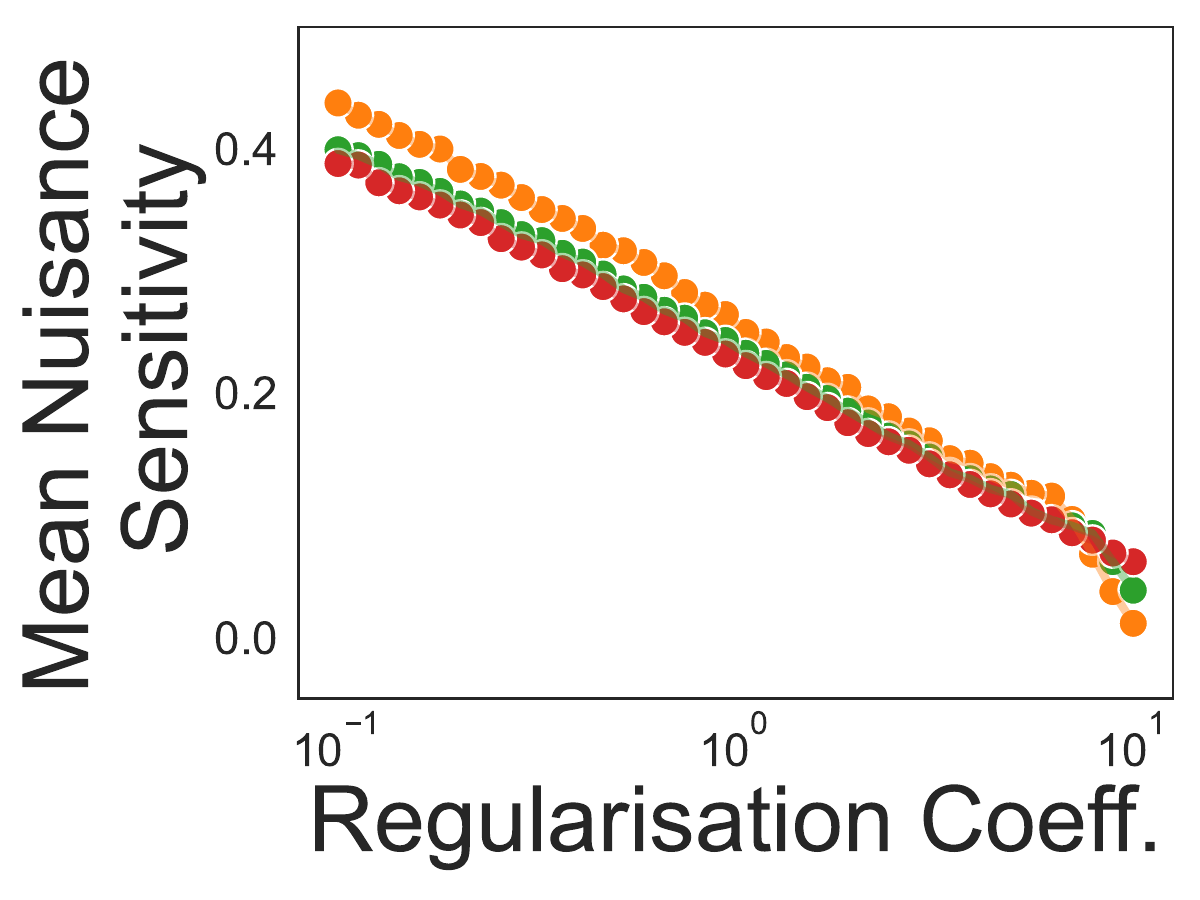}   
    \includegraphics[width=0.48\linewidth]{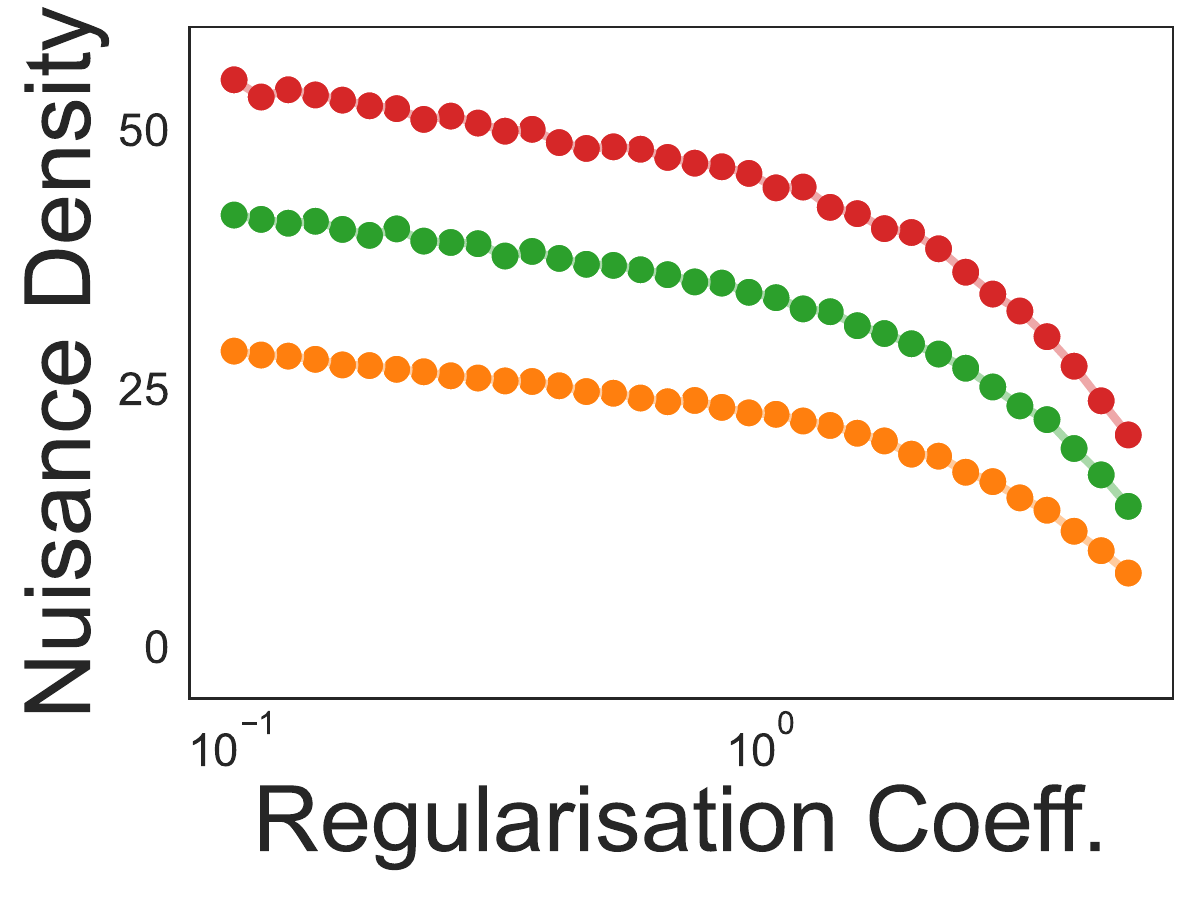}
   \caption*{Noisy}
\end{subfigure}\vline
\begin{subfigure}[t]{0.4\linewidth}
    \centering    \includegraphics[width=0.48\linewidth]{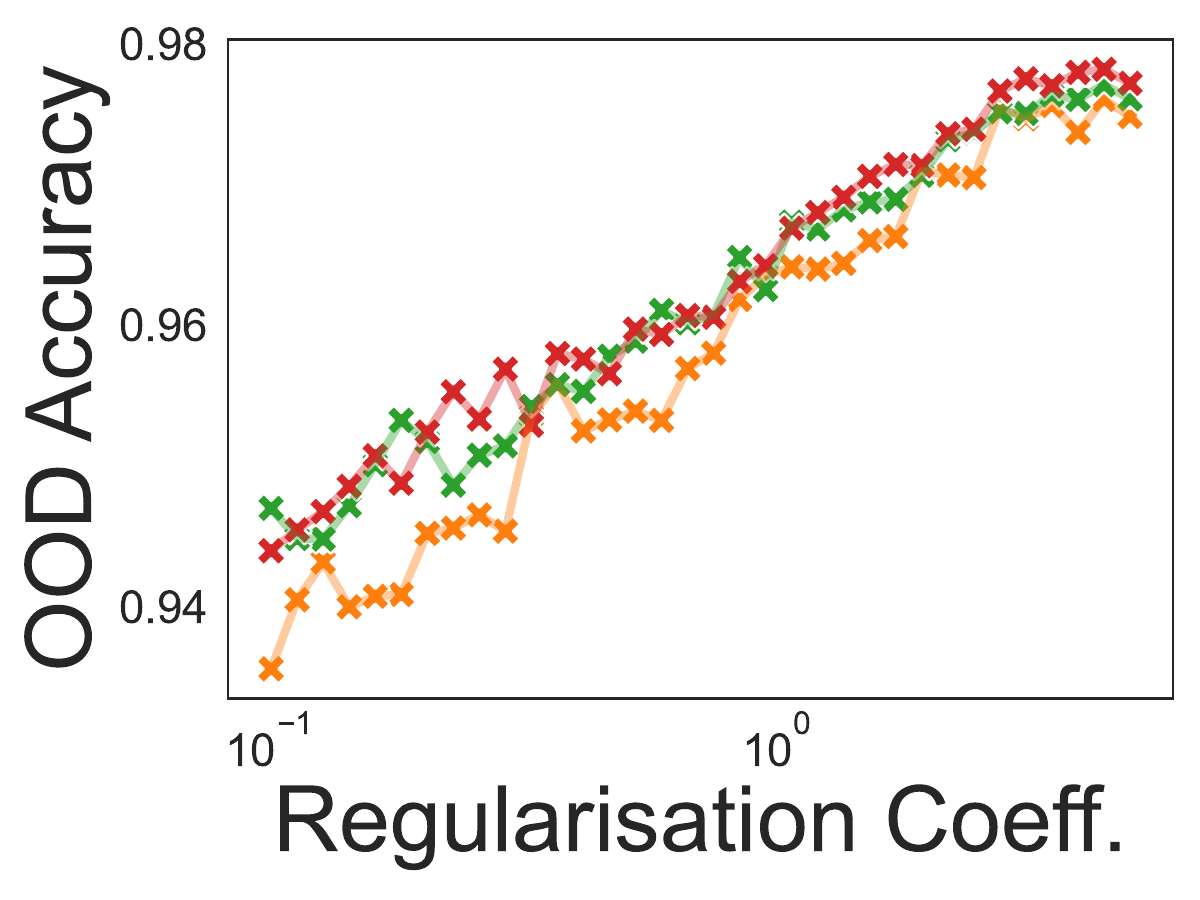}
     \includegraphics[width=0.48\linewidth]{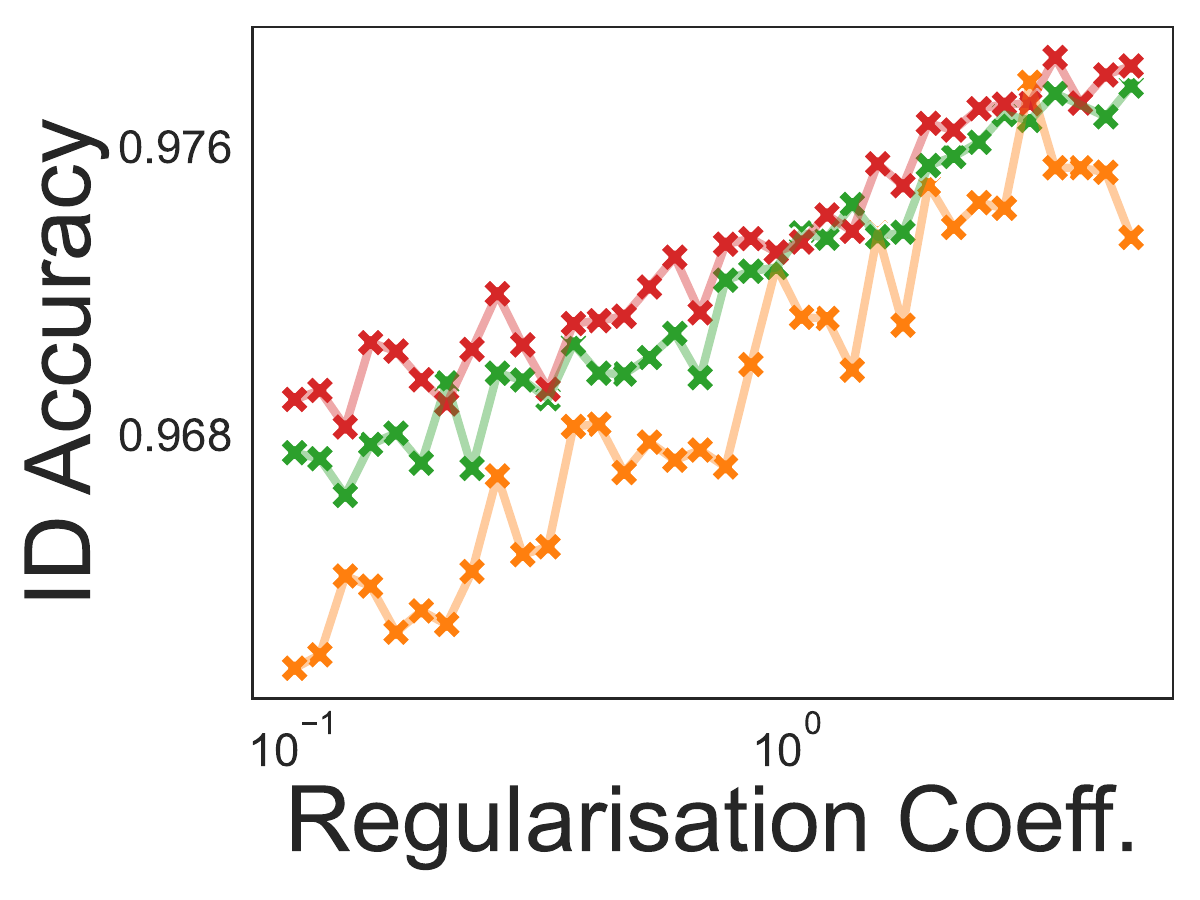}
    \includegraphics[width=0.48\linewidth]{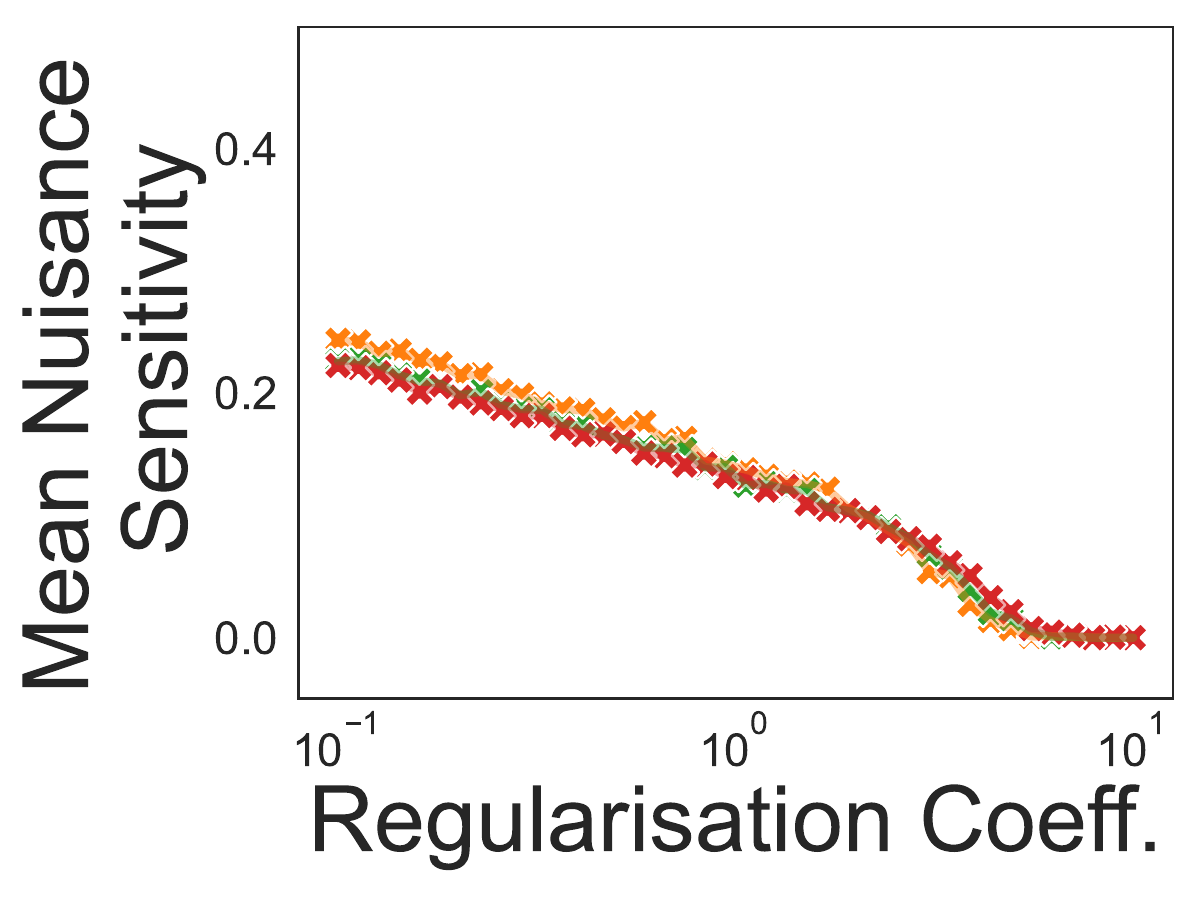}   
    \includegraphics[width=0.48\linewidth]{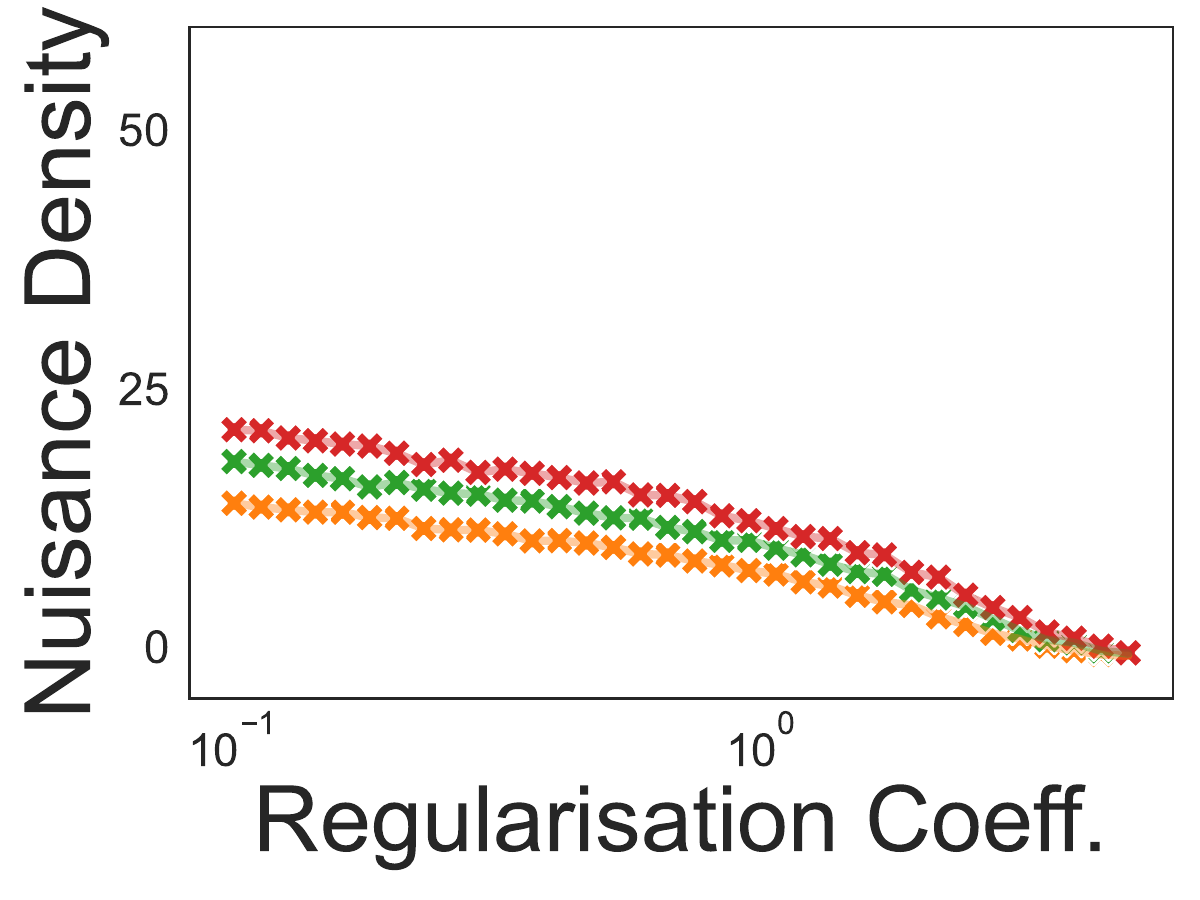}
    \caption*{Noiseless}
    \label{fig:enter-label}
\end{subfigure}
    \caption{We show varying the strength of regularisation impacts the~\Gls{id} and~\Gls{ood} accuracy as well as the spurious sensitivity and density. While both~\Gls{id} and~\Gls{ood} accuracy increases with increasing regularisation coefficient, larger datasets have a noticeably smaller~\Gls{ood} accuracy in the noisy setting for all regularisation strengths. In all other settings, larger datasets have a higher accuracy and this is the main factor leading to the \awline behaviour. Regarding spurious sensitivity and density, for sufficiently large regularisation both the spurious sensitivity and the density drops to zero much faster for the noiseless setting than the noisy setting.}
    \label{fig:vary-reg}
\end{figure}

\subsection{Colored MNIST dataset}
\label{app:cmnist}
As discussed in the main text, this dataset is derived from MNIST by introducing a color-based spurious correlation. Specifically, digits are initially assigned a binary label based on their numeric value (less than 5 or not), and this label is then corrupted with label noise probability \(\eta\). To make this set of experiment more realistic, we use an algorithm that is supposed to be robust to distribution shift. In particular, we construct domains one with \(0.35\) and one with \(0.7\) fraction of the samples with correlated label and colour. Then, we optimise an average of the losses on these two domains. For test set, the spurious correlation is at \(0.1\).

A three-layer MLP is then trained on this dataset to achieve zero training error by running the Adam optimizer for 1000 steps with \(\ell_2\) regularization. The width of the MLP is varied from \(16\) to \(2048\) to generate various runs. The learning rate is set at 0.001. The accuracy on the training distribution is referred to as the~\Gls{id} accuracy, and the accuracy on the test distribution is referred to as the~\Gls{ood} accuracy. Each set of runs~(multiple seeds etc) was run on a single GPU and took less than 30 minutes for the whole set.

\subsection{Functional Map of the World (fMoW) Dataset}
\label{app:fmow}

The original~\Gls{fmow} dataset~\citep{christie2018functional} contains satellite images from various parts of the world, classified into five geographical regions: Africa, Asia, America, Europe, and Oceania, and labeled according to one of 30 objects in the image. It also includes additional metadata regarding the time the image was captured. The~\Gls{fmow}-CS dataset is constructed by introducing a correlation between the domain and the label, \ie only sampling certain labels for certain domains. We use the domain-label pairing originally used by~\citet{shi2023how}, which ensures that if the dataset is sampled according to this pairing, the population of each class relative to the total number of examples remains stable. In this work, we use a spurious correlation level of 0.9, meaning 90\% of the training dataset follows the domain-label pairing, while the remaining 10\% does not match any domain-label pairs. The domain-label pairing for~\Gls{fmow}-CS is detailed in~\Cref{tab:fmow-c-dl}. To simplify the problem further, we only select five labels instead of all thirty, which is the first in each of the rows.

Similar to our previous experiments, we also introduce label noise with a probability of 0.5. For the~\Gls{ood} test data, we use the original WILDS~\citep{wilds2021,sagawa2022extending} test set for~\Gls{fmow}, which essentially creates a distribution shift by thresholding based on a timestamp; images before that timestamp are~\Gls{id} and images after are~\Gls{ood}. To obtain various training runs, we fine-tuned ImageNet pre-trained models, including ResNet-18, ResNet-34, ResNet-50, ResNet-101, and DenseNet121, with various learning rates and weight decays on the~\Gls{fmow}-CS dataset. We also varied the width of the convolution layers to increase the width of each network. In total, we trained nearly 400 models using various configurations where each model was trained on a single  48GB NVIDIA Quadro RTX 6000 with 36 CPUs or a 32GB NVIDIA V100 with 28 CPUs. Each run took between 9 hours and 15 hours depending on problem parameters.

\begin{figure}[!htb]
\centering
    \centering
    \includegraphics[width=0.49\linewidth]{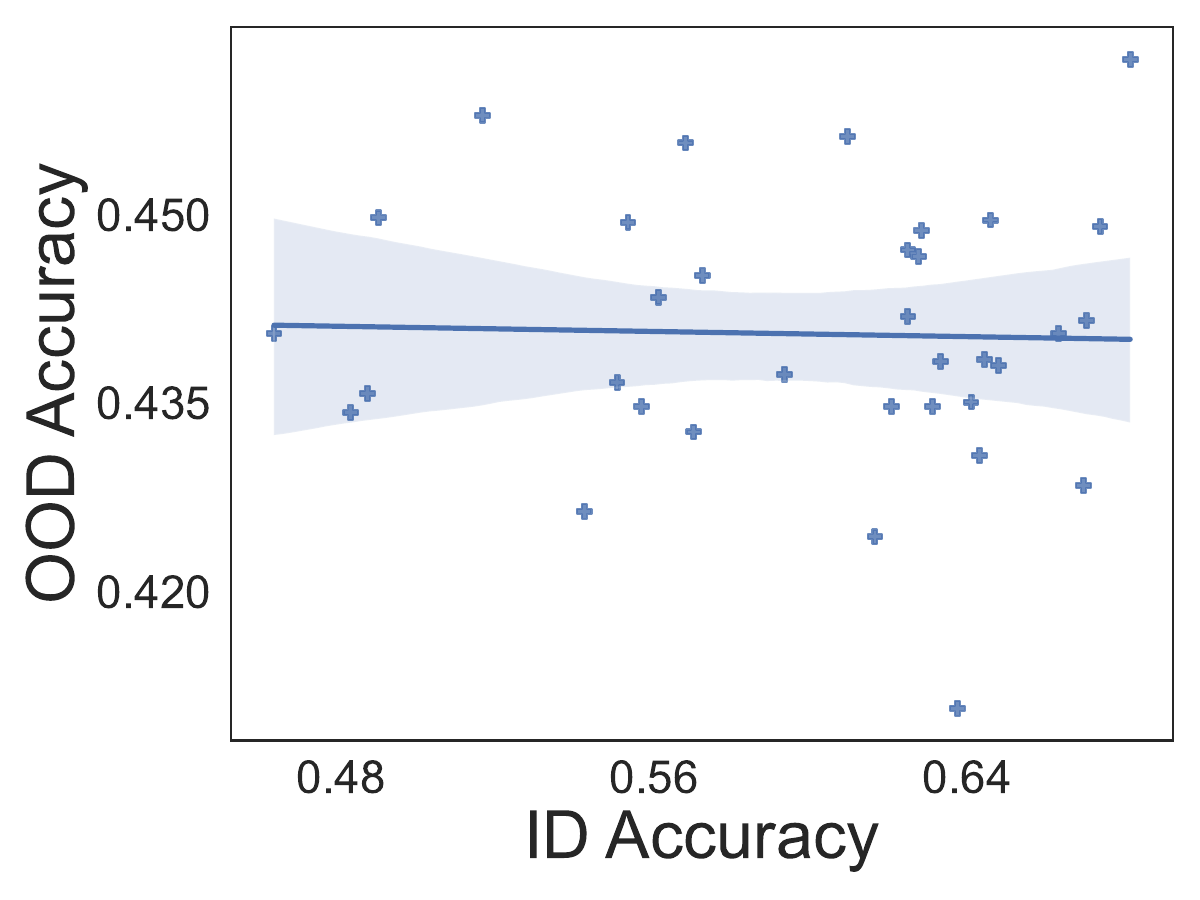}
    \caption*{Noisy no Spurious correlation}
\caption{Experiments on the \Gls{fmow} dataset without spurious correlation shows almost zero correlation between~\Gls{id} and~\Gls{ood} accuracy.}
\label{fig:fmow-exp-plot}
\end{figure}

\begin{table}[t]
\centering
\caption{Domain-label pairing for~\Gls{fmow}-CS.} 
\scalebox{0.78}{
\begin{tabular}{lc}    
    \toprule
    Domain (region) & Label \\ \midrule
    \multirow{4}{*}{Asia}  &  \multirow{4}{*}{\shortstack{Military facility, multi-unit residential, tunnel opening, \\ wind farm, toll booth, road bridge, oil or gas facility,\\ helipad, nuclear powerplant, police station, port}} \\
    & \\
    & \\
    & \\
    \multirow{4}{*}{Europe}  &  \multirow{4}{*}{\shortstack{Smokestack, barn, waste disposal, hospital, water  \\ treatment facility, amusement park, fire station, fountain, \\ construction site, shipyard, solar farm, space facility}} \\
    & \\
    & \\
    & \\
    \multirow{4}{*}{Africa}  &  \multirow{4}{*}{\shortstack{Place of worship, crop field, dam, tower, runway, airport, electric \\ substation, flooded road, border checkpoint, prison, archaeological site, \\factory or powerplant, impoverished settlement, lake or pond}} \\
    & \\
    & \\
    & \\
    \multirow{4}{*}{Americas}  &  \multirow{4}{*}{\shortstack{Recreational facility, swimming pool, educational institution, \\stadium, golf course, office building, interchange, \\ car dealership, railway bridge, storage tank, surface mine, zoo}} \\
    & \\
    & \\
    & \\
    \multirow{4}{*}{Oceania}  &  \multirow{4}{*}{\shortstack{Single-unit residential, parking lot or garage, race track, park, ground \\ transportation station, shopping mall, airport terminal, airport hangar,  \\ lighthouse, gas station, aquaculture, burial site, debris or rubble}} \\
    & \\
    & \\
    & \\
    \bottomrule
\end{tabular}
\label{tab:fmow-c-dl}}
\end{table}
\end{document}